\documentclass[final,runningheads,envcountsect,envcountsame]{llncs}

\newcommand\camerareadyversion[2]{#2}

\usepackage[T1]{fontenc}
\usepackage{graphicx} 
\usepackage{amsmath}
\usepackage{amssymb}
\usepackage{hyperref}
\usepackage{bm}
\usepackage{nicefrac}

\usepackage{xspace}
\usepackage{float}
\usepackage{enumitem}
\newcommand{\myparagraph}[1]{\smallskip\noindent \textit{#1}}
\usepackage{ marvosym }
\usepackage{ amssymb }
\usepackage{stmaryrd}
\usepackage{adjustbox}
\usepackage{booktabs}
\usepackage{pifont}
\usepackage{subcaption}
\usepackage{multirow}
\usepackage{microtype}
\usepackage[table]{xcolor}
\newcommand{\C}{\ensuremath{\mathcal{C}}\xspace}
\newcommand{\Sys}{\ensuremath{\mathcal{M}}\xspace}
\newcommand{\N}{\ensuremath{\mathcal{N}}\xspace}
\newcommand{\K}{\ensuremath{\mathcal{L}}\xspace}
\newcommand{\R}{\ensuremath{\mathcal{R}}\xspace}

\newcommand{\Obs}{\mathcal{T}}
\newcommand{\basis}{\ensuremath{B}}
\newcommand{\frontier}{\ensuremath{F}}
\newcommand{\Hyp}{\mathcal{H}}
\newcommand{\bigO}{\mathcal{O}}
\newcommand{\lsharp}{L^{\#}}
\newcommand{\alsharp}{AL^{\#}}
\newcommand{\soundlsharp}{L^{\#}_{\er,\mathtt{S}}}
\newcommand{\sclsharp}{L^{\#}_{\er,\mathtt{SC}}}
\newcommand{\lstar}{L^{\ast}}

\newcommand{\converges}{\ensuremath{\mathord{\downarrow}}}

\newcommand{\er}{\ensuremath{\mathsf{e}}\xspace}
\newcommand{\access}{\ensuremath{\mathsf{access}}}
\newcommand{\firsterror}[2]{\mathsf{first}_\mathsf{e}(#1, #2)}

\newcommand{\partialto}{\ensuremath{\rightharpoonup}}
\newcommand{\apart}{\ensuremath{\mathrel{\#}}}

\newcommand{\erlsharp}{\ensuremath{L^{\#}_{\er}}}
\newcommand{\erbasis}{\ensuremath{B}}
\newcommand{\erfrontier}{\ensuremath{F}}
\newcommand{\last}{\textsf{last}}
\newcommand{\sound}{\mathtt{S}}
\newcommand{\complete}{\mathtt{C}}
\newcommand{\moe}{\ensuremath{\mathsf{MoE}}}
\newcommand{\etal}{\emph{et al}.\xspace}

\newcommand{\accessT}{\ensuremath{\mathsf{access}^{\Obs}\hspace*{-0.05cm}}}
\newcommand{\matches}{\scalebox{0.9}{\ensuremath{{}\overset{\er}{=}{}}}}
\newcommand{\eralsharp}{\ensuremath{AL^{\#}_{\er}}}
\newcommand{\notmatches}{\scalebox{0.9}{\ensuremath{{}\overset{\er}{\neq}{}}}}

\usepackage{hyperref}
\usepackage{cleveref}
\Crefformat{figure}{#2Fig.~#1#3}
\Crefmultiformat{figure}{Figs.~#2#1#3}{ and~#2#1#3}{, #2#1#3}{ and~#2#1#3}

\usepackage{tikz}
\usepackage{pgfplots}
\pgfplotsset{width=10cm,compat=1.9}

\everymath=\expandafter{\the\everymath\displaystyle}

\usetikzlibrary{shapes.geometric, arrows}
\usetikzlibrary{automata,
                arrows.meta,
                chains,
                positioning,
                quotes,
                shapes.geometric,
                patterns,
                backgrounds}
\tikzset{
initial text=$ $,
treenode/.style = {align=center, inner sep=0pt, text centered},
knowledge/.style = {
    pattern=north east lines,
    pattern color=blue!70!black,
  },
hypothesis/.style = {
    pattern=north east lines,
    pattern color=orange!70!black,
  },
system/.style={
    pattern=north east lines,
    pattern color=black!40!white,
  },
basis/.style={
    pattern=north east lines,
    pattern color=magenta!70!black,
  },
frontier/.style={
    pattern=north east lines,
    pattern color=yellow!70!black,
  },
}
\usepackage[outline]{contour}
\contourlength{1.2pt}
\newcommand{\treeNodeLabel}[1]{\contour{white}{#1}}

\usepackage{algorithm}
\usepackage[noend]{algpseudocode}

\usepackage{mdframed}
\newenvironment{problemstatement}[1]
  {\mdfsetup{
    frametitle={\colorbox{white}{\space#1\space}},
    innertopmargin=-2pt,
    frametitleaboveskip=-\ht\strutbox,
    frametitlealignment=\center
    }
  \begin{mdframed}%
  }
  {\end{mdframed}}
  
\tikzset{
  do guard/.style={
    inner xsep=0pt,
  },
  guard line offset/.style={
    xshift=2mm,
  },
  bigtalloblong/.style={
    draw=black,
    minimum width=3pt,
    minimum height=1em,
    inner xsep=0pt,
  },
}
\newcommand{\connectDoGuards}[2]{%
  \draw[overlay,draw=black!40!white] ([guard line offset]#1) -- ([guard line offset]#2);
}

\algblockdefx{While}{EndWhile}[1]{
  \textbf{while}~\text{#1}
}{
  \textbf{end while}
}

\algblockdefx{DoIf}{EndDoIf}
[1]{%
  \begin{tikzpicture}[remember picture,baseline=(do.base)]
    \node[do guard] (do) {\textbf{do either}};
    \coordinate (guard line x) at ([guard line offset]do.west);
    \coordinate (last do) at (guard line x |- do.south);
  \end{tikzpicture}\\
  \begin{tikzpicture}[remember picture,baseline=(newdo.base)]
    \node[do guard] (newdo) {\phantom{\bfseries d}};
    \draw[draw=none] ([yshift=3pt]newdo.north);
    \node[bigtalloblong,anchor=center] (firstoblong) at (guard line x |- newdo.center) {};
    \coordinate (this guard top) at (guard line x |- newdo.north); 
    \connectDoGuards{last do}{this guard top} 
    \coordinate (last do) at (guard line x |- newdo.south);
  \end{tikzpicture}%
  ~#1 $\rightarrow$%
}
{%
  \begin{tikzpicture}[remember picture,baseline=(od.base)]
    \node[do guard] (od) {\textbf{end do}};
    \coordinate (end guard top) at (guard line x |- od.north); \connectDoGuards{last do}{end guard top}
  \end{tikzpicture}%
}

\algcblockdefx[DoIf]{DoIf}{ElsDoIf}{EndDoIf}
[1]{%
  \begin{tikzpicture}[remember picture,baseline=(newdo.base)]
    \node[do guard] (newdo) {\phantom{\bfseries d}};
    \draw[draw=none] ([yshift=3pt]newdo.north);
    \node[bigtalloblong,anchor=center] at (guard line x |- newdo.center) {}; 
    \coordinate (this guard top) at (guard line x |- newdo.north); 
    \connectDoGuards{last do}{this guard top} 
    \coordinate (last do) at (guard line x |- newdo.south);
  \end{tikzpicture}%
  ~#1 $\rightarrow$%
}
{%
  \begin{tikzpicture}[remember picture,baseline=(od.base)]
    \node[do guard] (od) {\textbf{end do}};
    \coordinate (end guard top) at (guard line x |- od.north); \connectDoGuards{last do}{end guard top}
  \end{tikzpicture}%
}

\newcommand{\StateIf}[1]{\State\textbf{if}~#1~\textbf{then:}}

\algloopdefx{IfLoop}[1]{\textbf{If} #1 \textbf{then}}

\newcommand{\defineApiFunction}[2][]{%
  \ifthenelse{\equal{#1}{}}{
    \expandafter\newcommand\csname #2\endcsname{\text{\upshape\textsc{%
          #2%
    }}\xspace}%
  }{
    \expandafter\newcommand\csname #2\endcsname{\text{\upshape\textsc{%
          #1%
        }}\xspace}%
  }%
}
 
\defineApiFunction{CheckConsistency}
\defineApiFunction{CheckSoundness}
\defineApiFunction{BuildHypothesis}
\defineApiFunction[EquivQuery]{EquivalenceQuery}
\defineApiFunction[ProcCounterEx]{ProcessCounterexample}
\defineApiFunction{OutputQuery}
\defineApiFunction{Ads}

\usepackage{todonotes}

\usepackage{bibnames}

\begin{document}

\title{Error-awareness Accelerates Active~Automata~Learning\thanks{This research is partially supported by the NWO grant No.~VI.Vidi.223.096.}}
%
\author{Loes Kruger 
\and  Sebastian Junges
\and Jurriaan Rot
  }
%
\authorrunning{L. Kruger et al.}
%
\institute{
 Institute for Computing and Information Sciences, \\
 Radboud University, Nijmegen, the Netherlands\\
 \email{\{loes.kruger,sebastian.junges,jurriaan.rot\}@ru.nl}
}

\maketitle

\camerareadyversion{}{\vspace{-0.6cm}}

\begin{abstract}
  Active automata learning (AAL) algorithms can learn a behavioral model of a system from interacting with it. 
  The primary challenge remains scaling to larger models, in particular in the presence of many possible inputs to the system. 
  Modern AAL algorithms fail to scale even if, in every state, most inputs lead to errors.
  In various challenging problems from the literature, these errors are observable, i.e., they emit a known error output.
  Motivated by these problems, we study learning these systems more efficiently.
  Further, we consider various degrees of knowledge about which inputs are non-error producing at which state. For each level of knowledge, we provide a matching adaptation of the state-of-the-art AAL algorithm $\lsharp$ to make the most of this domain knowledge.  
  Our empirical evaluation demonstrates that the methods accelerate learning by orders of magnitude with strong but realistic domain knowledge to a single order of magnitude with limited domain knowledge.
\end{abstract}

\section{Introduction} \label{sec:intro}
Active Automata Learning (AAL) can be used to infer models by analyzing the input-output behavior of a given System Under Learning (SUL). 
AAL has been used to learn a wide variety of systems including network protocols~\cite{DBLP:conf/uss/RuiterP15,DBLP:conf/cav/Fiterau-Brostean16,DBLP:conf/spin/Fiterau-Brostean17}, smart cards~\cite{DBLP:conf/woot/ChaluparPPR14}, legacy systems~\cite{DBLP:conf/ifm/SchutsHV16}, neoVim~\cite{DBLP:conf/spin/Ganty24} and Git version control~\cite{DBLP:conf/icst/MuskardinBTA24}. 
Libraries such as AALpy~\cite{DBLP:journals/isse/MuskardinAPPT22} and Learnlib~\cite{DBLP:conf/fmics/RaffeltSB05} implement common algorithms including $\lstar$~\cite{DBLP:journals/iandc/Angluin87}, TTT~\cite{DBLP:conf/rv/IsbernerHS14} and $\lsharp$~\cite{DBLP:conf/tacas/VaandragerGRW22} and support various types of automata.

Despite the significant progress in AAL, learning a complete model of a large real-world system is often not feasible within a day~\cite{DBLP:conf/icfem/SmeenkMVJ15,DBLP:conf/wcre/YangASLHCS19}. Systems with large input alphabets are particularly challenging~\cite{DBLP:journals/cacm/Vaandrager17,DBLP:conf/tacas/MalerM14,DBLP:conf/tacas/KrugerJR24}.
Traditionally, learning algorithms study the system response to \emph{every input} at \emph{every state}, yet in many states, most inputs are invalid or practically impossible. Often, such inputs immediately yield error outputs, which indicate that the inputs were invalid and the system is now stuck. 
For example, network protocols often contain a handshake phase. 
During this phase, only handshake-related inputs are allowed; all other inputs lead to an error. 
We have observed that considering all these state-input pairs causes significant overhead.
Therefore, we advocate using readily available domain knowledge about errors and error-producing inputs to reduce the overhead.
Specifically, we show how to incorporate various degrees of knowledge about error-producing inputs in both the hypothesis learning and the hypothesis testing steps, accelerating the state of the art in AAL.

We consider \emph{error-persistent} systems~\cite{DBLP:journals/corr/Yaacov}, i.e., systems where \emph{known} error outputs are followed only by error outputs.
Error persistence is common:
Once a connection is closed in a network protocol, further messages yield a response indicating this. Error persistence also applies to fatal errors, shutdowns, and out-of-resource messages.
Error persistence is also common when learning POMDP controllers~\cite{DBLP:conf/tacas/BorkCGKM24} or runtime monitors~\cite{DBLP:conf/atva/MaasJ25}, where some input sequences and their extensions are impossible. Intuitively, there is no need to learn the monitor response to teleporting cars, as long as cars cannot physically teleport.

Beyond error persistence, it is often reasonable to assume that some input words are known to lead to an error, cf.\ the teleporting cars. 
However, the \emph{complete} set of error-inducing input words is not always known. For instance, network protocol specifications can underspecify certain input words, leaving it unclear whether they produce errors.
We express domain knowledge as a regular language $\K$
describing the words that do \emph{not} lead to errors. 
We call $\K$ the \emph{reference} and distinguish three levels of knowledge, formalised as relations between $\K$ and the SUL $\Sys$: 
(1)~\emph{sound and complete}: a word is not in $\K$ if and only if it produces an error in $\Sys$,  
(2)~\emph{sound}: if a word is not in $\K$ then it produces an error in $\Sys$, and
(3)~\emph{arbitrary}: no formal relation.

As standard in AAL, we assume the learning algorithm can ask two types of queries about the SUL~\cite{DBLP:journals/iandc/Angluin87}:
Output queries (OQs, `What is the output of this word?') and equivalence queries (EQs, `Is the hypothesis equivalent to the SUL?').
Practically, we use EQs that are approximated using conformance testing~\cite{DBLP:journals/cacm/Vaandrager17,DBLP:conf/dagstuhl/HowarS16}. 
Conformance testing techniques construct test suites from a hypothesis. 
Executing these tests requires the majority of the interactions with the SUL~\cite{DBLP:journals/cacm/Vaandrager17}.
Given an error-persistent SUL, a reference, and an assumption about its relation to the SUL, our goal is to learn the model with few interactions. 

Towards this goal, we adapt conformance testing and learning based on refined notions of equivalence which, intuitively, are parameterized by the reference~$\K$. Specifically, in Sec.~\ref{sec:testing}, we truncate test sequences to either before or after an error occurs, depending on the availability of a sound reference. We prove completeness of these new test suites, relative to a suitable fault domain. Further, in Sec.~\ref{sec:ealearning}, we adapt the learning algorithm $\lsharp$~\cite{DBLP:conf/tacas/VaandragerGRW22}.
First, we present a version tailored to learning error-persistent systems, without assuming any additional reference. 
Then, we provide new algorithms for learning with a reference, one for each assumption about this reference.   
In particular, if $\K$ is sound, certain states can be distinguished without queries, if $\K$ is additionally complete, the number of necessary EQs  drops. For arbitrary references, we tailor adaptive AAL~\cite{DBLP:conf/fm/KrugerJR24}.

\begin{table}[htbp]
    \centering
    \caption{Overview of novel testing and learning algorithms.}
    \scalebox{.75}{
    \begin{tabular}{|p{1.8cm}p{0.6cm}|p{4.8cm}||p{1cm}p{0.6cm}|p{5.7cm}|}
        \hline
        Test suite & Sec. & Explanation & Alg. & Sec. & Explanation \\ \hline
         $T_\er$ & \ref{sec:test_perm}  & Test suite w persistent errors & $L_{\er}^{\#}$  & \ref{sec:lsharp_perm} & $\lsharp$ (\cite{DBLP:conf/tacas/VaandragerGRW22}) w\ persistent errors \\         $T_\sound$ & \ref{sec:sound_test} & Sound test suite &   $\eralsharp$  & \ref{sec:adap_lsharp_perm} & $\erlsharp$ + $\alsharp$ (\cite{DBLP:conf/fm/KrugerJR24}) w arbitrary reference \\
         $\mathsf{MoE}(T_\er, T_\sound)$ & \ref{sec:mab_test} & Distribution over test suites (\cite{DBLP:conf/tacas/KrugerJR24}) &  $\soundlsharp$ & \ref{sec:sound_alg} & $\erlsharp$ w\ sound reference     \\         
       & & &  $\sclsharp$ & \ref{sec:sound_compl_alg}  & $\erlsharp$ w\ sound and complete reference \\ \hline
    \end{tabular}
    }
    \label{tab:outline}
\end{table}

In summary, this paper contributes new test suites and variants of $\lsharp$, summarised in \Cref{tab:outline}. 
Each variant incorporates different forms of error-aware domain knowledge to accelerate learning and testing of error-persistent systems. Our experimental evaluation shows that such domain knowledge is realistically obtainable and significantly reduces the interactions with the SUL. \camerareadyversion{The full algorithms, proofs, benchmark details and results can be found in the appendix of the extended version of this paper~\cite{Kruger2026ErrorAwarenessAcceleratesArxiv}.}{The full algorithms, proofs, benchmark details and results can be found in the appendix.}
\section{Problem Statement} \label{sec:prelim}
In this section, we introduce the problem statement and necessary preliminaries.

\subsection{Automata}
We refer to the $i$-th element of a word $\sigma$ by $\sigma_i$ and use $\sigma_{:i}$ to indicate $\sigma_0 \cdot \sigma_1 \cdots \sigma_{i}$. 
Function $\last \colon \Sigma^* \to \Sigma$ returns the last symbol of a word. 
Concatenation of languages $A, B \subseteq \Sigma^*$ is given by $A \cdot B = \{ \sigma \cdot \mu\ |\ \sigma \in A, \mu \in B\}$. 
We write $f\colon X \partialto Y$ for a partial map, $f(x){\downarrow}$ if $f$ is defined for $x$, and $f(x){\uparrow}$ otherwise. 

\begin{definition}
    A \emph{Mealy machine} is a tuple $\Sys = (Q, q_0, I, O, \delta, \lambda)$ with finite sets $Q$ of \emph{states}, $I$ of inputs, $O$ of outputs, an \emph{initial state} $q_0 \in Q$, a \emph{transition function} $\delta\colon Q \times I \partialto Q$ and an \emph{output function} $\lambda\colon Q \times I \partialto O$.
\end{definition}
We lift the transition and output functions to words: for $i \in I$ and $\sigma \in I^*$,  $\delta(q, i\cdot \sigma) = \delta(\delta(q,i),\sigma)$ and $\lambda(q, i\cdot \sigma) = \lambda(q,i) \cdot \lambda(\delta(q,i),\sigma)$. We define $|\Sys|$ to be the number of states in $\Sys$ and use superscript $\Sys$ to show to which Mealy machine we refer, e.g., $Q^\Sys$. We abbreviate $\delta(q_0,\sigma)$ by $\delta(\sigma)$. We assume that Mealy machines are complete and minimal, unless specified otherwise. 

\begin{definition} \label{def:ap}
    Given a language $\K \subseteq I^*$ and Mealy machines $\Sys$ and $\N$, states $p \in Q^\Sys$ and $q \in Q^\N$ are $\K$\emph{-equivalent}, written as $p \sim_\K q$, if $\lambda^\Sys(p,\sigma)=\lambda^\N(q,\sigma)$ for all $\sigma \in \K$. 
    States $p, q$ are \emph{equivalent}, written $p \sim q$, if they are $I^*$-equivalent. States $p, q$ are \emph{apart}, written $p \apart q$, if there is some $\sigma\in I^*$ such that $\lambda(q,\sigma){\downarrow}$, $\lambda(p,\sigma){\downarrow}$ and $\lambda(q,\sigma)\neq\lambda(p,\sigma)$. 
\end{definition} 
Mealy machines $\Sys$ and $\N$ are $\K$-equivalent if $q_0^{\Sys} \sim_\K q_0^{\N}$. 
A prefix-closed language $P \subseteq I^*$ is a \emph{state cover} for a Mealy machine $\Sys$ if, for every state $q \in Q^\Sys$, there exists $\sigma \in P$ s.t.\ $\delta(q_0,\sigma)=q$. A state cover $P$ is \emph{minimal} if $|P| = |Q^\Sys|$.
A language $W_q \subseteq I^*$ is a \emph{state identifier} for $q \in Q^\Sys$ if  for all $p \in Q^\Sys$ with $p \apart q$ there exists $\sigma \in W_q$ with $\sigma \vdash p \apart q$ . A \emph{separating family} is a collection of state identifiers $\{ W_q\}_{q \in Q^\Sys}$ s.t.\ for all 
$p \apart q$, there exists $\sigma \in W_p \cap W_q$ with $\sigma \vdash p \apart q$.
We fix $\mathsf{sep}(W,p,q)$ to uniquely represent one such~$\sigma$.

\subsection{Learning and Testing}
We assume our learning algorithm can perform output queries (OQs) on the System Under Learning (SUL) $\Sys$. This means that when the algorithm poses an \textsc{OQ} with $\sigma \in I^*$ on $\Sys$, $\Sys$ returns $\lambda^{\Sys}(q_0,\sigma)$.
Traditionally, a learning algorithm can pose an equivalence query (EQ) to a teacher: the teacher replies \texttt{yes} if the provided hypothesis Mealy machine $\Hyp$ is equivalent to the SUL, and \texttt{no} with a counterexample witnessing inequivalence otherwise.
When no teacher is available, conformance testing can be used to approximate the EQ. 
Conformance testing constructs a test suite from a hypothesis Mealy machine $\Hyp$~\cite{DBLP:journals/cacm/Vaandrager17,DBLP:conf/dagstuhl/HowarS16}, and each sequence in the suite is posed as an OQ. Ideally, test suites contain a failing test for every alternative model of the SUL.

\begin{definition} \label{def:compl}
    A \emph{fault domain} is a set of Mealy machines $\C$. A test suite $T \subseteq I^*$ is complete for $\Hyp$ w.r.t. $\C$ if for all $\Sys \in \C,~\Hyp \sim_T \Sys$ implies $\Hyp \sim \Sys$. 
\end{definition}
Since equivalence is checked on output sequences, prefixes of test sequences add no additional information. Accordingly, all our examples use test suites with only maximal elements.
For any finite test suite $T$ a Mealy machine can be constructed that is inequivalent to $\Hyp$ but produces the same outputs as $\Hyp$ on every sequence in $T$. Therefore, test suites are constructed such that they provide guarantees such as $k$-completeness~\cite{MooreGedanken1956} or $k$-$A$-completeness~\cite{DBLP:conf/concur/VaandragerM25}. For example, a test suite $T$ is $k$-complete for $\Hyp$ if it is complete w.r.t.\ all models with at most $|\Hyp|+k$ states.

\subsection{Problem Statement}
In this paper, we focus on Mealy machines where after executing an input that produces an error, all following outputs are also errors.

\begin{definition}
    A Mealy machine $\Sys$ is $\er$-\emph{persistent} if there exists a special output $\er \in O$ such that for all $\sigma \in I^*$ and $i \in I$, $\last(\lambda^\Sys(\sigma)) = \er$ implies $\last(\lambda^\Sys(\sigma\cdot i)) = \er$. A language $\K \subseteq I^*$ is $\er$-persistent if for all $\sigma \in I^*$ and $i \in I$, $\sigma \notin \K$ implies $\sigma\cdot i \notin \K$.
\end{definition}
Instead of a single output, \er may be an alias for several outputs that are linked to erroneous behavior. We assume that the set of outputs representing $\er$ is given. 
Moreover, we may have domain knowledge in the form of an $\er$-persistent DFA that rejects some or all words that produce errors and call such a reference DFA \emph{sound}, respectively, \emph{sound and complete}.

\begin{definition}
    Let $\Sys$ be a Mealy machine and $\K \subseteq I^*$ a regular language. $\K$ is \emph{sound} for \Sys if for all $\sigma \in I^*$, $\sigma \notin \K$ implies $\last(\lambda^{\Sys}(\sigma))=\er$. $\K$ is \emph{complete} for \Sys if the implication holds the other way around.
\end{definition}
\begin{figure*}[t]
    \centering
    \begin{subfigure}[t]{0.42\textwidth}
        \scalebox{0.9}{\begin{tikzpicture}[->,>=stealth',shorten >=1pt,auto,node distance=1.2cm,main node/.style={circle,draw,font=\sffamily\large\bfseries},
  ]
  \def\xoffset{8mm}
  \node[initial,state,system] (0) {\treeNodeLabel{$q_0$}};
  \node[state,system] (1) [below of=0, yshift=-0.5cm] {\treeNodeLabel{$q_1$}};
  \node[state,system] (2) [right of=1, xshift=0.5cm] {\treeNodeLabel{$q_2$}};
  \node[state,system] (3) [above of=2, yshift=0.5cm] {\treeNodeLabel{$q_{\er}$}};
  \node[state,system] (4) [right of=3, yshift=-1cm] {\treeNodeLabel{$q_3$}};
  
  \path[every node/.style={font=\sffamily\scriptsize}]
  (0) edge[] node[above, align=center] {$+/\er$} (3)
      edge[] node[left] {$h/\checkmark$} (1)
  (1) edge[] node[above, align=center,xshift=-0.2cm] {$+/\er$} (3)
      edge[] node[below] {$k/\checkmark$} (2)
  (2) edge[] node[left, align=center] {$+/\er$} (3)
      edge[] node[below, align=center, xshift=0.3cm] {$d/\checkmark$} (4)
  (3) edge[loop right] node[right] {$+/\er$} (3)
  (4) edge[] node[below, align=left, xshift=-0.2cm] {$+/\er$} (3)
      edge[loop right] node[right, align=center] {$d/$\Letter} (4)
      ;
  \end{tikzpicture}}
        \caption{Toy TLS system $\Sys$.}
        \label{fig:1a}
    \end{subfigure}
    \hfill
    \begin{subfigure}[t]{0.27\textwidth}
        \scalebox{0.9}{\begin{tikzpicture}[->,>=stealth',shorten >=1pt,auto,node distance=1.2cm,main node/.style={circle,draw,font=\sffamily\large\bfseries},
  ]
  \def\xoffset{8mm}
  \node[initial,state,knowledge,accepting] (0) {\treeNodeLabel{$r_0$}};
  \node[state,knowledge,accepting] (1) [below of=0, yshift=-0.5cm] {\treeNodeLabel{$r_1$}};
  \node[state,knowledge,accepting] (2) [right of=1, xshift=0.5cm] {\treeNodeLabel{$r_2$}};
  
  \path[every node/.style={font=\sffamily\scriptsize}]
  (0) edge[] node[left] {$h$} (1)
  (1) edge[] node[below] {$k$} (2)
      edge[loop left] node[left, align=center] {$h$} (1)
  (2) edge[loop above] node[above, align=center] {$h, d$} (3)
      ;
  \end{tikzpicture}}
        \caption{Reference $\K_0$.}
        \label{fig:1d}
    \end{subfigure}
    \hfill
    \begin{subfigure}[t]{0.27\textwidth}
        \scalebox{0.9}{\begin{tikzpicture}[->,>=stealth',shorten >=1pt,auto,node distance=1.2cm,main node/.style={circle,draw,font=\sffamily\large\bfseries},
  ]
  \def\xoffset{8mm}
  \node[initial,state,knowledge,accepting] (0) {\treeNodeLabel{$p_0$}};
  \node[state,knowledge,accepting] (1) [below of=0, yshift=-0.5cm] {\treeNodeLabel{$p_1$}};
  \node[state,knowledge,accepting] (2) [right of=1, xshift=0.5cm] {\treeNodeLabel{$p_2$}};
  
  \path[every node/.style={font=\sffamily\scriptsize}]
  (0) edge[] node[left] {$h$} (1)
  (1) edge[] node[below] {$k$} (2)
  (2) edge[loop above] node[above] {$d$} (3)
      ;
  \end{tikzpicture}}
        \caption{Reference $\K_1$.}
        \label{fig:1e}
    \end{subfigure}
    \vspace{0.2cm}
    ~\\
    \begin{subfigure}[t]{0.31\textwidth}
        \centering
        \scalebox{0.9}{\begin{tikzpicture}[->,>=stealth',shorten >=1pt,auto,node distance=1.2cm,main node/.style={circle,draw,font=\sffamily\large\bfseries},
  ]
  \def\xoffset{8mm}
  \node[initial,state,knowledge,accepting] (0) {\treeNodeLabel{$s_0$}};
  \node[state,knowledge,accepting] (1) [below of=0, yshift=-0.5cm] {\treeNodeLabel{$s_1$}};
  \node[state,knowledge,accepting] (2) [right of=1, xshift=0.5cm] {\treeNodeLabel{$s_2$}};
  \node[state,knowledge,accepting] (3) [above of=2, yshift=0.5cm] {\treeNodeLabel{$s_3$}};
  
  \path[every node/.style={font=\sffamily\scriptsize}]
  (0) edge[] node[left] {$h$} (1)
  (1) edge[] node[below] {$k$} (2)
  (2) edge[loop right] node[right, align=center] {$k$} (2)
  (2) edge[] node[left] {$d$} (3)
      ;
  \end{tikzpicture}}
        \caption{Reference $\K_2$.}
        \label{fig:1f}
    \end{subfigure}
    \hfill
    \begin{subfigure}[t]{0.22\textwidth}
        \centering
        \scalebox{0.9}{\begin{tikzpicture}[->,>=stealth',shorten >=1pt,auto,node distance=1.2cm,main node/.style={circle,draw,font=\sffamily\large\bfseries},
  ]
  \def\xoffset{8mm}
  \node[initial,state,hypothesis] (0) {\treeNodeLabel{$h_0$}};
   \node[state,hypothesis] (1) [below of=0, yshift=-0.5cm] {\treeNodeLabel{$h_\er$}};

  \path[every node/.style={font=\sffamily\scriptsize}]
  (0) edge[] node[right, align=center] {$h/\checkmark$\\$+/\er$} (1)
  (1) edge[loop right] node[right, align=center] {$+/\er$} (1)
      ;
  \end{tikzpicture}}
        \caption{Hypothesis $\Hyp$.}
        \label{fig:hyp}
    \end{subfigure}
    \begin{subfigure}[t]{0.45\textwidth}
        \centering 
        \scalebox{0.9}{\begin{tikzpicture}[->,>=stealth',shorten >=1pt,auto,node distance=1.5cm,main node/.style={circle,draw,font=\sffamily\large\bfseries},
  ]
  \def\xoffset{8mm}
  \node[initial,state,basis] (0) {\treeNodeLabel{$q_0$}};
  \node[state,basis] (1) [right of=0] {\treeNodeLabel{$q_1$}};
  \node[state,system] (2) [right of=1] {\treeNodeLabel{$q_2$}};
  \node[state,frontier] (3) [above right of=1] {\treeNodeLabel{$q_3$}};
  \node[state,system] (4) [right of=3] {\treeNodeLabel{$q_4$}};
  \node[state,system] (5) [above right of=0] {\treeNodeLabel{$q_5$}};
  
  \path[every node/.style={font=\sffamily\scriptsize}]
  (0) edge[] node[above] {$h/\checkmark$} (1)
      edge[] node[above left] {$d/\er$} (5)
  (1) edge[] node[above,xshift=-0.2cm] {$k/\checkmark$} (3)
      edge[] node[above] {$h/\er$} (2)
  (3) edge[] node[above] {$d/\checkmark$} (4)
      ;
  \end{tikzpicture}}
        \caption{Observation tree $\Obs$.}
        \label{fig:obs}
    \end{subfigure}
    \caption{Mealy machine, DFA references and an observation tree related to a toy TLS system. Double circle states represent final states. Omitted DFA transitions lead to the (non-final) sink state. The $+$-symbol on transitions represents all inputs for which a transition is not explicitly drawn.}
    \label{fig:overview}
\end{figure*}
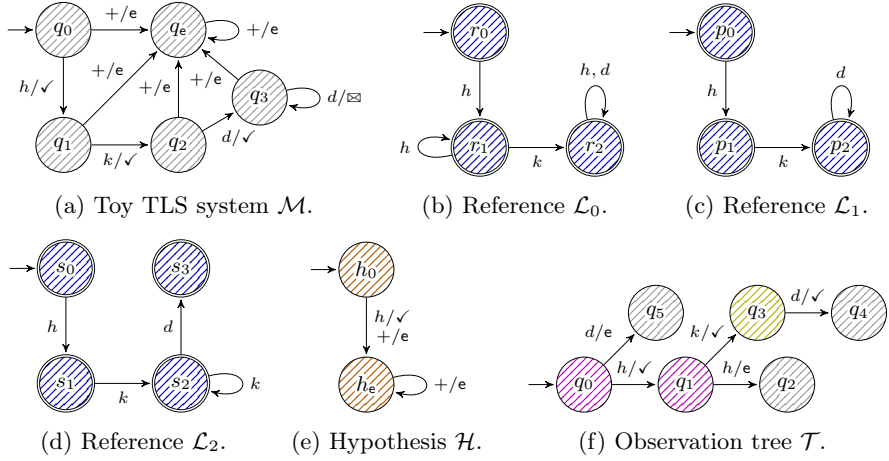
\begin{example}
    Figure~\ref{fig:1a} shows a Mealy machine representing a toy version of the TLS protocol with $O = \{\checkmark$, \Letter, $\er\}$ and $I = \{h, k, d, c\}$ representing \textit{hello}, \textit{kex}, \textit{data} and \textit{close} respectively. 
    The DFAs \Cref{fig:1d}, \Cref{fig:1e}, and \Cref{fig:1f} represent a sound, a sound and complete, and an unsound and incomplete reference, respectively.
    \Cref{fig:1d} is not complete as $hh \in \K_0$. \Cref{fig:1f} is neither sound nor complete as $hkdd \notin \K_2$ and $hkk \in \K_2$. 
    Reference $\K_2$ can still be useful because most accepted words are non-error-producing.
\end{example}
In this paper, we focus on these \emph{assumptions} and use them to learn a Mealy machine more efficiently, specifically. We solve the following problem:

\begin{problemstatement}{Error-aware AAL with a Reference} \label{problem}
    Given an $\er$-persistent SUL $\Sys$, a sound and complete/sound/arbitrary reference language $\K$, efficiently learn a Mealy machine $\Hyp$ s.t.\ $\Sys \sim \Hyp$.
\end{problemstatement}

\section{Error-aware Testing}
\label{sec:testing}
In this section, we prune test suites using $\er$-persistence and error-aware references. 

\subsection{Testing with Persistent Errors} \label{sec:test_perm}
We show that if an SUL $\Sys$ and a Mealy machine $\Hyp$ are both $\er$-persistent, test words can be truncated after the first $\er$ output without losing the power of the test suite. 
Let $\firsterror{\Hyp, \sigma}$ be the smallest $i$ where $\lambda^{\Hyp}(\sigma)_{i}\neq\er \land \lambda^{\Hyp}(\sigma)_{i+1} = \er$ and $\infty$ if no such $i$ exists. The \emph{filter} $f_\er$ returns a truncated word:
\begin{equation}\label{eq:w}
  f_\er(\Hyp, \sigma)=
  \begin{cases}
    \sigma_{:i+1} & \quad\text{if } i = \firsterror{\Hyp}{\sigma} < \infty. \\    \sigma & \quad\text{otherwise.}
  \end{cases}
\end{equation} 

\begin{definition} \label{def:ertest}
Given a test suite $T$ and a Mealy machine $\Hyp$, the $\er$-\emph{persistent test suite} is $T_\er(\Hyp) = \{ f_\er(\Hyp, \sigma)\ |\ \sigma \in T\}$.
\end{definition}

\begin{example} \label{ex:testsuite}
    For test suite $T = \{kh, dh, ch,$ $hhh, hkh, hdh, hch\}$ and hypothesis $\Hyp$ in Fig.~\ref{fig:hyp}, the $\er$-persistent test suite is $T_{\er}(\Hyp)= \{k, d, c, hh, hk, hd, hc\}$.
\end{example}
We refer to $T_\er(\Hyp)$ by $T_\er$, whenever possible. As we assume the SUL is $\er$-persistent, we consider fault domains with only $\er$-persistent Mealy machines.
\begin{definition} \label{def:kaercompl}
Let $\C$ be a fault domain. Then the $\er$-\emph{persistent fault (sub)domain} is $\C_{\er} = \C \cap \{ \Sys \mid \Sys \text{ is } \er\text{-persistent} \}$.
\end{definition}

\begin{theorem} \label{lem:error_test_suite}
Let $\Hyp$ be an $\er$-persistent Mealy machine and $\C$ a fault domain.
If test suite $T$ is complete for $\Hyp$ w.r.t.\ $\C$, then $T_\er$ is complete for $\Hyp$ w.r.t.\ $\C_{\er}$. 
\end{theorem}
When $\C$ is clear from the context, we may call $T_\er$ \emph{$\er$-complete}.

\subsection{Testing with Reference Models} \label{sec:sound_test}
Given a reference that is sound for both the SUL and $\Hyp$, a test word may be truncated to the longest prefix in $\K$. Unlike $f_\er$, which truncates after observing $\er$, a sound reference allows truncation to just before $\er$ is produced.
Let $\firsterror{\K}{\sigma}$ be the smallest $i$ where $\sigma_{:i} \in \K \land \sigma_{:i+1} \notin \K$ and $\infty$ if no such $i$ exists. 
If $\firsterror{\K}{\sigma} < \infty$, then the filter $f_\sound$, defined below, cuts the word accordingly. 
If the resulting word ends with multiple error in $\Hyp$, it is further cut by $f_{\er}$. 

\begin{equation}\label{eq:w2}
  f_{\sound}(\Hyp,\K,\sigma)=
  \begin{cases}
    \sigma_{:i} & \quad\text{if } i = \firsterror{\K}{\sigma} < \infty \land  i \leq \firsterror{\Hyp}{\sigma}. \\
    f_{\er}(\sigma) & \quad\text{otherwise.}
  \end{cases}
\end{equation}

\begin{definition} 
Given a test suite $T$, a Mealy machine $\Hyp$ and a reference $\K$, the \emph{sound test suite} is $T_\sound(\Hyp, \K) = \{ f_\sound(\Hyp,\K,\sigma)\ |\ \sigma \in T\}$.
\end{definition}
\begin{example}
    Consider $T$ and $\Hyp$ from Ex.~\ref{ex:testsuite}. For sound and complete reference~$\K_1$ from Fig~\ref{fig:overview}, the sound test suite is $T_{\sound}(\Hyp,\K_1) = \{hk\}$.
\end{example}
If the Mealy machine $\Sys$ and the hypothesis $\Hyp$ are both $\er$-persistent, the reference $\K$ is sound for both $\Sys$ and $\Hyp$ and $T$ is complete for some fault domain $\C$, then the sound test suite  $T_{\sound}(\Hyp,\K)$ is $\er$-complete.
\begin{theorem} \label{lem:sound_test_suite}
Let $\Hyp$ be an $\er$-persistent Mealy machine and $\C$ a fault domain.
If a test suite $T$ is complete for $\Hyp$ w.r.t.\ $\C$ and reference $\K$ is sound for all $\Sys \in \C \cup \{ \Hyp \}$, then $T_\sound(\Hyp,\K)$ is complete w.r.t.\ $\C_{\er}$. 
\end{theorem}
 \label{sec:mab_test}
If the reference $\K$ is \emph{not sound} for the SUL, performing only the tests in $T_\sound$ does not guarantee $\er$-completeness. However, the test suite $T_\sound$ probably still contains counterexamples and is smaller than $T_\er$. Therefore, we combine $T_\sound$ and $T_\er$ using the Mixture of Experts ($\moe$) approach~\cite{DBLP:conf/tacas/KrugerJR24}. 
The $\moe$ approach takes a set of test suites and maintains a confidence value for each test suite based on how often that test suite produced a counterexample during previous testing rounds. 
During testing, $\moe$ selects the next test suite either at random or based on the confidence score, then generates the next test word from the chosen test suite. 
The algorithm terminates when all test suites have been fully executed or a counterexample has been found. We use $\moe(T_\er(\Hyp),$ $T_\sound(\Hyp,\K))$ to refer to the $\moe$ approach with the $T_\sound$ and $T_\er$ test suites.

\begin{lemma} \label{lem:mostly_test_suite}
Let $\Hyp$ be an $\er$-persistent Mealy machine, $\C$ a fault domain, and $\K$ a reference. If test suite $T$ is complete for $\Hyp$ w.r.t.\ $\C$, then $\moe(T_\er(\Hyp),$ $T_\sound(\Hyp,\K))$ is complete w.r.t.\ $\C_\er$.
\end{lemma} 

\section{Error-aware Learning} \label{sec:learning}
\label{sec:ealearning}
Below, we introduce four variations of the state-of-the-art AAL algorithm $\lsharp$~\cite{DBLP:conf/tacas/VaandragerGRW22}. Each variant matches an assumption about a reference or its absence (cf.~\Cref{tab:outline}).

\subsection{$\lsharp$ with Persistent Errors} \label{sec:lsharp_perm}
 We present \erlsharp, a novel adaptation of $\lsharp$ that assumes an $\er$-persistent  SUL. 
For readers familiar with $\lsharp$, the main changes are: (1)~the frontier excludes states with errors and (2)~OQs are posed per symbol and stop after an error.

\begin{definition} \label{def:obs_tree}
    A Mealy machine $\Obs$ is a \emph{tree} if there is a unique word $\sigma \in I^*$ for each $q \in Q^{\Obs}$ such that $\delta^{\Obs}(q_0,\sigma) = q$. A tree $\Obs$ is an \emph{observation tree} for a Mealy machine $\Sys$ if there exists mapping $m \colon Q^\Obs \to Q^\Sys$ such that $m(q_0^{\Obs}) = q_0^\Sys$ and if there is a transition $\delta^\Obs(q,i)=q'$, then there is a transition $\delta^\Sys(m(q),i)=m(q')$ and $\lambda^\Obs(q,i)=\lambda^\Sys(m(q),i)$. 
\end{definition}

\noindent $\erlsharp$ uses an observation tree to store the observed traces.
 In an observation tree $\Obs$, the basis $\erbasis$ is a subtree of $\Obs$ consisting of distinct states of the hidden Mealy machine $\Sys$. All states in $\basis$ are pairwise apart, i.e., for all $q \neq q' \in \basis$ it holds that $q \apart q'$. 
Initially, $\basis$ contains only the root state, i.e., $\basis = \{q_0\}$. 
For a fixed basis $\erbasis$, the frontier $\erfrontier \subset Q^{\Obs}$ consists of all immediate successors of basis states which are not basis states themselves and are not reached by $\er$-producing inputs, i.e., $\erfrontier = \{ r \mid \delta^{\Obs}(q, i) = r \land \lambda^{\Obs}(q, i) \neq \er \text{ for some } q \in \basis, r \notin \basis \text{ and } i \in I\}$.

\begin{example}
    The tree $\Obs$ in \Cref{fig:obs} is an observation tree for $\Sys$ in Figure~\ref{fig:1a}. Its basis $\erbasis$ is $\{q_0, q_1\}$ as $h \vdash q_0 \apart q_1$ and its frontier $\erfrontier$ is  $\{q_3\}$. State $q_2$ is not in the frontier, as it is reached with an $\er$-producing transition.
\end{example}

\noindent A frontier state $r \in \erfrontier$ is \emph{isolated} if it is apart from all $q \in \erbasis$ and \emph{identified} with a basis state $q \in \erbasis$ if it is apart from all $q' \in \erbasis$ except $q$. 
The basis is \emph{adequate} if $\delta(q,i){\downarrow}$ for all $q \in B, i \in I$, and all states $r \in \erfrontier$ are identified. When the basis is adequate, the observation tree may be folded into an $\er$-persistent Mealy machine hypothesis using $\textsc{BuildHypothesis}_{\er}$.
This construction folds back frontier states to their identified basis state and directs transitions from basis states that produce $\er$ to a special sink state $q_\er$ if such transitions exist. 

\camerareadyversion{}{\clearpage}

\begin{definition} \label{build_hyp}
    Given an observation tree $\Obs$ with basis $\erbasis$ and frontier $\erfrontier$, $\textsc{BuildHypothesis}_{\er}$ constructs a hypothesis Mealy machine $\Hyp =$ $(\erbasis \,\cup\, {q_\er}, q_0,$ $I, O,$ $\delta', \lambda')$\footnote{If there exists some $b \in \erbasis$ such that $\delta(b,i)=b$ and $\lambda(b,i)=\er$ for all $i \in I$, we redirect all transitions from $q_\er$ to $b$ and remove the afterwards unreachable $q_\er$.}. Functions $\delta'$ and $\lambda'$ are constructed as follows.
    \begin{equation*}
        \delta'(q,i), \lambda'(q,i) = \begin{cases}
            q_\er, \er & \text{if } \lambda^{\Obs}(q,i)=\er \lor \delta^\Obs(q,i){\uparrow} \lor q = q_\er\\
            b, \lambda^{\Obs}(q,i) & \text{if } \delta^{\Obs}(q,i)=b \in \erbasis\\
            b, \lambda^{\Obs}(q,i) & \text{if } \delta^{\Obs}(q,i)=r \in \erfrontier \land \exists b \in B \text{ s.t. } \neg(b \apart r)\\
        \end{cases}
    \end{equation*}
\end{definition}

Moreover, we use $\textsc{OQ}_{\er}$ to pose inputs one by one and stop after an error $\er$ is observed. We write $\access(q)$ to indicate the word leading to $q$ in $\Obs$.
The $\erlsharp$ algorithm repeatedly applies the following rules until the correct hypothesis is found:\looseness=-1
\begin{itemize}[noitemsep,topsep=0pt]
    \item The \emph{promotion rule} adds $q \in \erfrontier$ to the basis $\basis$ if $q$ is isolated.
    \item The \emph{extension rule} poses $\textsc{OQ}_{\er}(\access(q) i)$ for $q \in \basis$, $i \in I$ if $\delta(q,i){\uparrow}$.
    \item The \emph{separation rule} takes a state $r \in \erfrontier$ that is not apart from $q, q' \in B$ and poses $\textsc{OQ}_{\er}(\access(r) \sigma)$ with $\sigma \vdash q \apart q'$, resulting in $q \apart r$ or $q \apart r'$.
    \item The \emph{equivalence rule} builds the hypothesis $\Hyp$ if $\erbasis$ is adequate. The procedure checks whether $\Obs$ and $\Hyp$ agree on traces and then poses $\textsc{EQ}(\mathcal{H})$, handling a counterexample when necessary and terminating the algorithm if $\Hyp$ is equivalent to $\Sys$. Only performing an EQ when $\erbasis$ is adequate ensures each counterexample leads to a new basis state.
\end{itemize}
The equivalence rule ensures that the algorithm terminates only once the correct model is learned. A SUL $\Sys$ is considered \emph{learned} if $\Hyp \sim \Sys$, where $\Hyp$ is the hypothesis produced by the algorithm. Correctness thus reduces to proving termination.
Each rule application increases a norm, analogous to $\lsharp$~\cite{DBLP:conf/tacas/VaandragerGRW22}. 
\begin{lemma} \label{lem:error_learning}
    Let $\Sys$ be an $\er$-persistent SUL with $n$ states and $k$ inputs. $\erlsharp$ learns $\Sys$ within $\bigO(kn^2 + n \log m)$ OQs and at most $n - 1$ EQs, where $m$ is the length of the longest counterexample returned by the EQs.
\end{lemma}

\subsection{Learning with a Reference} \label{sec:adap_lsharp_perm}
Next, we introduce $\eralsharp$, a so-called adaptive variant of $\erlsharp$. Traditionally, adaptive AAL methods incorporate domain knowledge in the form of a reference \emph{Mealy machine} via \emph{rebuilding}~\cite{DBLP:conf/ifm/DamascenoMS19,DBLP:conf/birthday/FerreiraHS22} and \emph{state matching}~\cite{DBLP:conf/fm/KrugerJR24}. 
Instead, we assume a minimal \emph{DFA} representation $\R$ for the reference language $\K$, thereby explicitly relying on the regularity of $\K$. We borrow the notation from Mealy machines for DFAs.
Compared to $AL^{\#}$~\cite{DBLP:conf/fm/KrugerJR24}, we use tailored notions of rebuilding and matching and builds upon \erlsharp. The main differences are clarified below. 

\subsubsection*{Rebuilding.}
Let $P$ be a state cover and $W$ a separating family for $\R$. We introduce two rules that extend \erlsharp from Sec.~\ref{sec:lsharp_perm}. The \emph{rebuilding rule} aims to show apartness between every $q \in \erbasis$ and some $r \in \erfrontier$ to find a state in $\R$ that is not contained in $\erbasis$ yet.
The prioritized promotion rule ensures that $\access(q) \in P$ for every $q \in \erbasis$, which allows applying the rebuilding rule more often. 
\begin{example}
    Suppose we learn $\Sys$ using $\K_2$ (from \Cref{fig:overview}). Initially, $Q^\Obs = \{q_0\}$. Let $q = q' = q_0$, $i = h$ and $k \vdash s_0 \apart s_1$. Because $\varepsilon, h \in P$ and $\delta^{\Obs}(t_0,h){\uparrow}$, the rebuilding rule is applied and poses $\textsc{OQ}_{\er}(hk)$ and $\textsc{OQ}_{\er}(k)$ to isolate $\delta^{\Obs}(q_0,h)$.
\end{example}
We formally define the rebuilding rules as follows:
\begin{itemize}[noitemsep,topsep=0pt]
    \item Consider states $q, q' \in \erbasis$ and $i \in I$. If ($\delta(q,i){\uparrow}$ or $\delta(q,i)\in \erfrontier$) and $\access(q')$, $\access(q)i \in P$, then the \emph{rebuilding rule} poses queries $\textsc{OQ}_{\er}(\access(q) i \sigma)$ with $\sigma = \mathsf{sep}(W,\delta^{\R}(\access(q')),\delta^{\R}(\access(q)i))$. We only apply this rule if $\sigma$ \emph{guarantees progress}. Some $\sigma \in I^*$ guarantees progress if there exists maximal prefixes $\rho$ of $i\sigma$ and $\rho'$ of $\sigma$ such that $\delta^{\Obs}(q,\rho){\downarrow}$, $\delta^{\Obs}(q',\rho'){\downarrow}$ and either $(\delta^\Obs(q,i\sigma){\uparrow} \land \delta^{\Obs}(q,\rho)\neq\er)$ or $(\delta^\Obs(q',\sigma){\uparrow} \land \delta^{\Obs}(q',\rho')\neq\er)$.    
    \item If $q \in \frontier$ is isolated and $\access(q) \in P$,   \emph{prioritized promotion} adds $q$ to~$\erbasis$.
\end{itemize}

\subsubsection{Matching.}
Rebuilding heavily relies on the chosen $P$ to discover  states in $\R$. Thus, if $\R$ diverges from the SUL (only) in the initial state, rebuilding may fail to discover states.
State matching, on the other hand, looks at the residual language of states in $\R$ and aims to match basis states with these languages.
The matching relation $\mathsf{mdeg} \colon Q^{\Obs} \times Q^{\R} \to \mathbb{R}$ presented in~\cite{DBLP:conf/fm/KrugerJR24} matches states based on the outputs of the words defined for the basis state. 
In contrast, our matching relation measures how often the basis state and $\K$ agree on whether a word yields an error.
We define $\mathsf{WI}(q) = \{(w,i) \in I^* \times I \mid \delta^\Obs(q,wi){\downarrow}\}$ as the defined words from state $q \in Q^\Obs$. We define
\[
\mathsf{mdeg}(b,q) = 
    \displaystyle \frac{|\{ (w,i) \in \mathsf{WI}(b) \mid \lambda^{\Obs}(\delta^\Obs(b,w),i) = \er \leftrightarrow  \delta^{\R}(q,wi) \notin Q_F^{\R} \}|}{|\mathsf{WI}(b)|}, 
\]
using $\nicefrac{0}{0} = 0$. States $b \in \basis$ and $q \in Q^{\R}$ \emph{match}, written $b \matches q$, if for all $q' \in Q^{\R}$ with $q' \neq q$, $\mathsf{mdeg}(b,q) \geq \mathsf{mdeg}(b,q')$. We use this matching relation in the match separation and match refinement rules.
\footnote{Prioritized separation, as used in $AL^{\#}$, is omitted here because it prioritizes witnesses in $W$ which only concern $\er$-production, $\Obs$ likely contains shorter witnesses.}
\begin{example}
    Consider state $q_1$ in $\Obs$ (\Cref{fig:obs}) and reference $\K_1$ (\Cref{fig:1e}). State $q_1$ matches $p_1$ as $\mathsf{mdeg}(q_1,p_1)=\nicefrac{|\{k,h\}|}{|\{k,h\}|}=1$, $\mathsf{mdeg}(q_1,p_0)=\nicefrac{|\{h\}|}{|\{k,h\}|}=\nicefrac{1}{2}$ and $\mathsf{mdeg}(q_1,p_2)=\nicefrac{|\emptyset|}{|\{k,h\}|}=0$. 
    Because $q_1 \matches p_1$, $\delta^\Obs(q_1,k) \in \erfrontier$ and  $\delta^{\R}(p_1,k)$ does not match any $b \in \erbasis$, we apply match separation to try to show apartness between $q_3$ and $q_1$ with $\sigma = k$ as $\lambda^\Obs(q_1,k)\neq \er \land \delta^{\R}(hk,k) \notin Q_F^{\R}$. We pose $\textsc{OQ}_\er(hkk)$ which leads to isolation of $q_3$.
\end{example}
We formally define the matching rules as follows:
\begin{itemize}[noitemsep,topsep=0pt]
    \item Consider basis states $q, q'$, reference state $p, p'$ and $i \in I$. If there exists $\delta^{\Obs}(q,i) = r \in \erfrontier$, $\delta^{\R}(p,i) = p'$ such that $\neg(q' \apart r)$, $p$ matches $q$ and $p'$ has no matches, then the \emph{match separation rule} poses $\textsc{OQ}_{\er}(\access(q)i\sigma)$ where $\sigma$ is chosen such that $\last(\lambda^\Obs(q',\sigma)) = \er \leftrightarrow \delta^{\R}(p',\sigma) \in Q_F^{\R}$.     
    \item If frontier state $q$ matches $p, p' \in Q^\R$, the \emph{match refinement rule} poses $\textsc{OQ}_{\er}(\access(q)i\sigma)$ with $\sigma = \mathsf{sep}(W,p,p')$.
\end{itemize}
The $\eralsharp$ algorithm (listed in \Cref{alg:persistent_errors_alsharp}) consists of $\erlsharp$ and the above-mentioned rules.
Checking for undiscovered states in $\R$ increases complexity compared to $\erlsharp$, but often reduces the number of EQs in practice.

\begin{lemma} \label{lem:mostly_learning}
    Let $\Sys$ be an $\er$-persistent SUL with $n$ states and $k$ inputs. Let $\K$ an $\er$-persistent reference, canonically represented by a DFA with $o$ states. $\eralsharp$ learns $\Sys$ within $\bigO(kn^2 + no^2 + n \log m)$ OQs and at most $n - 1$ EQs, where $m$ is the length of the longest counterexample returned by any EQ.

\end{lemma}

\subsection{Learning with a Sound Reference} \label{sec:sound_alg}
Previously, we did not make any assumption about the quality of the reference. We now assume the reference is sound. In a nutshell, this assumption allows a stronger notion of apartness, relying on witnesses whose existence is implied by the reference instead of only relying on witnesses in the observation tree.
\begin{definition} \label{def:s_ap}
    Let $\Obs$ be an observation tree. 
    States $q, p \in Q^\Obs$ are \emph{$\sound$-apart} w.r.t.\ a sound reference $\K$, written $p \apart_{\sound,\K} q$, if $\sigma\in I^*$ exists such that $\sigma \vdash p \apart q$ or $\last(\lambda^{\Obs}(p,\sigma)) \neq \er$ and $\access^{\Obs}(q)\sigma \notin \K$ (or symmetrically with $p$ and $q$ swapped).
\end{definition}
The definitions of isolated and identified are lifted to $\apart_{\sound,\K}$. The basis is $\sound$-adequate if all $r \in \erfrontier$ are identified and $\delta(q,i){\downarrow}$ for all $q \in B, i \in I$ with $\access^\Obs(q)i \in \K$.

\begin{example}
    Consider reference $\K_1$ from Fig.~\ref{fig:1d} and $\Obs$ from \Cref{fig:obs}. States $q_0, q_4$ are $\sound$-apart with witness $h$, as $\lambda^{\Obs}(q_0,h)\neq \er$, while $\access^\Obs(q_4) \cdot h= hkdh \notin \K_1$. 
\end{example}
We adapt $\erlsharp$ in four main ways to take advantage of a sound reference:
\begin{enumerate}[noitemsep,topsep=0pt]
    \item We use $\textsc{OQ}_{\sound}$ which poses inputs one by one and stops when the following input leads to a rejected state in $\K$ or after an error $\er$ is observed.
    \item We use $\sound$-apartness in the $\erlsharp$ rules.
    \item We adopt the rebuilding procedure of $\eralsharp$ to rebuild the sound reference. 
    \item \Cref{lem:sound_test_suite} and the correctness of the algorithm rely on $\K$ being sound for $\Hyp$. We therefore check this assumption explicitly before posing an EQ restricted to language $\K$. If this check fails, then there must exist some counterexample $\sigma \in I^*$ such that $\sigma \notin \K$ and $\last(\lambda^\Hyp(q_0,\sigma)) \neq \er$. We perform $\textsc{OQ}_{\sound}(\sigma,\K)$ to add the counterexample to $\Obs$ and process it using \textsc{ProcCounterEx}$_\sound$.
\end{enumerate} 
$\soundlsharp$, listed in~\Cref{alg:sound_lsharp}, has the same OQ and EQ complexity as $\lsharp$ and learns a Mealy machine that is correct w.r.t. $\K$.

\begin{theorem} \label{thm:sound_learning}
    Let $\Sys$ be an $\er$-persistent SUL and $\K$ an $\er$-persistent reference. 
    Let $n, k$ and $m$ be as in Lemma~\ref{lem:error_learning}.
     $\soundlsharp$ requires $\bigO(kn^2 + n\ log\ m)$ OQs and at most $n - 1$ EQs to learn the smallest Mealy machine $\Hyp$ such that $\Hyp \sim_{\K} \Sys$ and $\lambda^{\Sys}(\sigma) = q_{\er}$ for all $\sigma \notin \K$.
\end{theorem}

When a reference is sound and the learned Mealy machine is consistent with the SUL on all non-error producing words, then $\soundlsharp$ learns the correct Mealy machine as all other words give output $\er$. 

\begin{corollary}
    If $\K$ is a sound reference for $\Sys$, then $\soundlsharp$ learns a Mealy machine $\Hyp$ such that $\Hyp \sim \Sys$.
\end{corollary}

\camerareadyversion{}{\clearpage}
\begin{algorithm}[!htbp]
	\begin{algorithmic}
		\Procedure{L$^{\#}_{\er,\sound}$}{$\K$}
      \State $\R_\K \gets$ minimal DFA representation of $\K$
      \State $P \gets$ minimal state cover of $\R_\K$ 
      \State $W \gets$ separating family of $\R_\K$

      \DoIf{$(\delta^{\Obs}(q,i) \in \erfrontier \lor \delta^{\Obs}(q,i){\uparrow}), \neg(q' \apart_{\sound,\K} \delta^{\Obs}(q,i)),$ $\access^{\Obs}(q)i, \access^{\Obs}(q') \in P$,\\
      \qquad\qquad$\sigma = \mathsf{sep}(W,\delta^{\R_\K}(\access^{\Obs}(q)i),\delta^{\R_\K}(\access^{\Obs}(q')))$,\\ \qquad\qquad $\sigma$ guarantees progress w.r.t. soundness\footnotemark\\
      \qquad\qquad for some $q, q' \in \erbasis, i \in I$}
      \Comment{rebuilding}
      \State $\textsc{OQ}_{\sound}(\access^{\Obs}(q)i\sigma,\K)$
      \State $\textsc{OQ}_{\sound}(\access^{\Obs}(q')\sigma,\K)$
      \ElsDoIf{$q$ isolated, for some $q \in \erfrontier$ with $\access^{\Obs}(q) \in P$}
      \Comment{prioritized promotion}
      \State $\basis \gets \basis \cup \{ q \}$
      \EndDoIf
		\DoIf{$q$ isolated, for some $q \in \erfrontier$}
      \Comment{promotion}
      \State $\basis \gets \basis \cup \{ q \}$
		\ElsDoIf{$\delta^{\Obs}(q,i){\uparrow}, \access^{\Obs}(q)i\in \K$, for some $q \in \basis, i \in I$}
      \Comment{extension}
      \State $\textsc{OQ}_{\sound}(\mathsf{access}(q) \; i, \K)$
		\ElsDoIf{$\neg(q\apart_{\sound,\K} r)$, $\neg(q\apart_{\sound,\K} r')$, $r\neq r'$\\
      \qquad\qquad for some $q \in \erfrontier$ and $r,r'\in \erbasis$}
      \Comment{separation}
      \State $\sigma \gets \text{witness of \(r\apart_{\sound,\K} r'\)}$
      \State $\textsc{OQ}_{\sound}(\mathsf{access}(q) \; \sigma,\K)$

		\ElsDoIf{$\erbasis$ is $\sound$-adequate}
      \Comment{equivalence}
		  \State $\Hyp \gets \BuildHypothesis_{\er}$ 
		  \State $(b, \sigma) \gets \CheckConsistency(\Hyp)$	
        \If{$b = \texttt{yes}$}
         \State $(b, \rho) \gets \CheckSoundness(\K,\Hyp)$
            \If{$b = \texttt{yes}$}
            \State $(b, \rho) \gets \EquivalenceQuery(\Hyp)$ \Comment{equivalence on $\K$}
            \If{$b = \texttt{yes}$} 
               \State \Return $\Hyp$
            \EndIf
            \EndIf
            \State $\textsc{OQ}_{\sound}(\rho,\K)$
            \State $\sigma \gets$ shortest prefix of $\rho$ such that
            $\delta^{\Hyp}(q_0^\Hyp, \sigma) \apart_{\sound,\K} \delta^{\Obs}(q_0^\Obs, \sigma)$
         (in $\Obs$)
         \EndIf
		  \State $\textsc{ProcCounterEx}_\sound(\Hyp, \sigma)$ \Comment{using $\apart_{\sound,\K}$}
    \EndDoIf
		\EndProcedure
	\end{algorithmic}
	\caption{$\erlsharp$ for a sound reference. }
	\label{alg:sound_lsharp}
\end{algorithm}
\footnotetext{Some $\sigma \in I^*$ guarantees progress w.r.t. soundness if there exists maximal prefixes $\rho$ of $i\sigma$ and $\rho'$ of $\sigma$ such that $\delta^{\Obs}(q,\rho){\downarrow}$, $\delta^{\Obs}(q',\rho'){\downarrow}$ and either $(\access^{\Obs}(q)i\sigma \in \K \land \delta^\Obs(q,i\sigma){\uparrow} \land \delta^{\Obs}(q,\rho)\neq\er)$ or $(\access^{\Obs}(q')\sigma \in \K \land \delta^\Obs(q',\sigma){\uparrow} \land \delta^{\Obs}(q',\rho')\neq\er)$.}

\subsection{Learning with a Sound and Complete Reference} \label{sec:sound_compl_alg}
Next, we discuss optimising $\erlsharp$ for utilising sound and complete references.
\begin{definition} \label{def:sc_ap}
    Let $\Sys$ be an $\er$-persistent Mealy machine. 
    States $q, p \in Q^\Obs$ are \emph{$\complete$-apart} w.r.t.\ a complete reference $\K$, written $p \apart_{\complete,\K} q$, if $\sigma\in I^*$ exists such that $\sigma \vdash p \apart q$ or $\last(\lambda^{\Obs}(p,\sigma)) = \er$ and $\access^{\Obs}(q)\sigma \in \K$ (or symmetrically with $p$ and $q$ swapped).
\end{definition}
Two states are $\sound\complete$-apart if they are $\sound$-apart or $\complete$-apart. We write $p \apart_{\sound\complete,\K} q$ to indicate that states $p$ and $q$ are $\sound\complete$ apart w.r.t. reference $\K$.
We lift the definition of isolated and identified states to $\sound\complete$-apartness.

Contrary to $\soundlsharp$, we do not use the rebuilding procedure of $\eralsharp$. Because we assume references are minimized, the SUL must have at least as many distinct states as the sound and complete reference. 
Therefore, we construct the test suite $T_\K$ for $\R$ with $P$ a minimal \emph{accepting} state cover and $W$ a separating family for $\R$. Here, an accepting state cover of a DFA is a state cover that contains a word $\sigma$ such that $\delta(q_0,\sigma) = q$ for every reachable state $q \in Q_F$.
\begin{equation}\label{eq:t}
    T_\K = \{ piw \mid p \in P, i \in I^{\leq 1}, w \in W_{\delta^{\R}(pi)} \}
\end{equation} 
A test suite such as $T_\K$ is naturally contained in a $k$-complete test suite with $k \geq 0$ generated by the HSI-method~\cite{LuoTesting1994,PetrenkoTesting1993}. After executing $T_\K$, we set the basis to $\erbasis = \{ \delta^{\Obs}(p) \mid p \in P \}$ as all states reached by $P$ are pairwise apart. 
\begin{lemma} \label{lem:rebuilding}
    Let $\Sys$ be a Mealy machine, $\K$ a sound and complete reference for $\Sys$, $P$ a minimal accepting state cover and $T_\K$ as in \Cref{eq:t}.
    Consider the observation tree $\Obs$ that arises when executing $T_\K$ using $\textsc{OQ}_{\sound}$, then $P$ forms a basis, i.e., all states reached by $p \in P$ are pairwise $\sound\complete$-apart.
\end{lemma}
We adapt $\erlsharp$ in three main ways to exploit a sound and complete reference:
\begin{enumerate}[noitemsep,topsep=0pt]
    \item Similar to $\soundlsharp$, we use $\textsc{OQ}_{\sound}$ instead of $\textsc{OQ}_\er$.
    \item We execute $T_\K$ and set $\erbasis$ to the states reached by $P$ before any of the rules.
    \item We use $\sound\complete$-apartness in the $\erlsharp$ rules.
\end{enumerate}
The $\sclsharp$ algorithm (listed in \Cref{alg:sound_complete_lsharp}) has the same OQ complexity as $\lsharp$ but requires fewer EQs whenever reference \K has more than two states. 

\begin{theorem} \label{thm:soundcomplete_learning}
 Let $\Sys$ be an $\er$-persistent SUL and $\K$ an $\er$-persistent reference. 
    Let $n, k$ and $m$ be as in Lemma~\ref{lem:error_learning}.
    If $\K$ is complete w.r.t.\ $\Sys$, then $\sclsharp$ requires $\bigO(kn^2 + n\ log\ m)$ OQs and at most $n - 1$ EQs to learn the smallest Mealy machine $\Hyp$ such that $\Hyp \sim_{\K} \Sys$ and $\lambda^{\Hyp}(\sigma) = q_{\er}$ for all $\sigma \notin \K$. 
\end{theorem}

\begin{lemma} \label{lem:nmino}
    If $\K$ is a sound and complete reference for $\Sys$, then $\sclsharp$ learns a Mealy machine $\Hyp$ such that $\Hyp \sim \Sys$ with at most $n - o$ EQs, with $o$ as defined in \Cref{lem:mostly_learning}.
\end{lemma}

\section{Experimental Evaluation} \label{sec:experiments}
In this section, we empirically investigate the performance of the algorithms $\erlsharp$, $\eralsharp$, $\soundlsharp$ and $\sclsharp$
proposed in Sec.~\ref{sec:ealearning} and compare it to the baseline $\lsharp$~\cite{DBLP:conf/tacas/VaandragerGRW22}. 
The source code and benchmarks are available online~\cite{ErrorAwareZenodo}.\footnote{\url{https://gitlab.science.ru.nl/lkruger/errorawarelearningaalpy}}
We investigate the following questions in three experiments:
\begin{description}
    \item[RQ1] 
    Do $\erlsharp$ and $\sclsharp$ with a sound and complete reference improve over $\lsharp$?
    \item[RQ2] Do $\eralsharp$ and $\soundlsharp$ with (sound) references improve over $\erlsharp$?
    \item[RQ3] What happens when $\soundlsharp$ and $\sclsharp$ use unsound or incomplete references?
\end{description}
Additionally, we discuss whether the used references can be obtained realistically.

\medskip\noindent\textbf{Benchmarks}
We consider benchmarks from the literature representing network protocol implementations, lithograph systems or monitors:

\myparagraph{TLS.} 
Transport Layer Security (TLS) is a common security
protocol of which several implementations have been learned by de Ruiter~\cite{DBLP:conf/uss/RuiterP15}.
We use 18 models of this benchmark with between 6-16 states and 8-11 inputs. These models should adhere to the TLS1.2 specification~\cite{dierks2008transport}. 

\myparagraph{ASML.} 
The ASML benchmark contains 33 Mealy machines representing components in lithography systems for the semiconductor industry first released for the 2019 RERS challenge~\cite{DBLP:conf/tacas/JasperMMSHSSHSK19}. These models are hard to learn due to their size~\cite{DBLP:conf/tacas/KrugerJR24,DBLP:conf/wcre/YangASLHCS19}. We limit ourselves to 11 models used in~\cite{DBLP:conf/tacas/KrugerJR24} with under 10000 transitions. 

\myparagraph{Monitors.}
We consider Mealy machine representations of runtime monitors for Hidden Markov Models (HMMs) used by~\cite{DBLP:conf/atva/MaasJ25}. These monitors flag potentially dangerous states when moving randomly through a partially observable state space. 
We restrict our analysis to the HMMs from~\cite{DBLP:conf/atva/MaasJ25} with the smallest environment parameters and select the two smallest or only horizon values, resulting in 14 models with between 21-75 states and 18-325 inputs. 

\myparagraph{Realistic references.}
For the TLS and ASML models, sound and complete references are not realistically obtainable as it is unlikely that precise enough specifications exist. However, for TLS or other non-proprietary network protocols, it is possible to create sound, but quite cautious, references by inspecting the open-source specifications. For the monitor benchmark, sound and complete references \emph{are realistic} as they can be extracted from the underlying HMM. 

\medskip\noindent\textbf{Setup}
We adapted the algorithms for $\lsharp$ and $AL^{\#}$ implemented in learning library AALpy~\cite{DBLP:journals/isse/MuskardinAPPT22}. 
To test a hypothesis $\Hyp$, we apply an adapted version of state-of-the-art randomized Wp-method~\cite{DBLP:conf/icfem/SmeenkMVJ15,DBLP:conf/models/GarhewalD23} that truncates test words as in $T_\er(\Hyp)$ and $T_\sound(\Hyp,\K)$. We run all experiments with 30 seeds. We measure the performance in the number of symbols (the sum of the length $+ 1$ of all OQs, including OQs asked in testing). We enforce a total symbol budget of $10^6$. We omit the final EQ when the hypothesis and target are equivalent as standard. Finally, we cache OQs during learning and testing. \camerareadyversion{Appendix C of the extended version of this paper~\cite{Kruger2026ErrorAwarenessAcceleratesArxiv} contains details on algorithm parameters, the references, and full benchmark results.}{\Cref{app:models} contains details on algorithm parameters, the references, and full benchmark results.}

\subsection{RQ1: Sound and Complete vs Error-aware vs Baselines}
We evaluate the performance of $\erlsharp$ and $\sclsharp$ with a sound and complete against the baseline. We extract a sound and complete reference $\K$ for each (known) SUL. 

\begin{figure}[t]
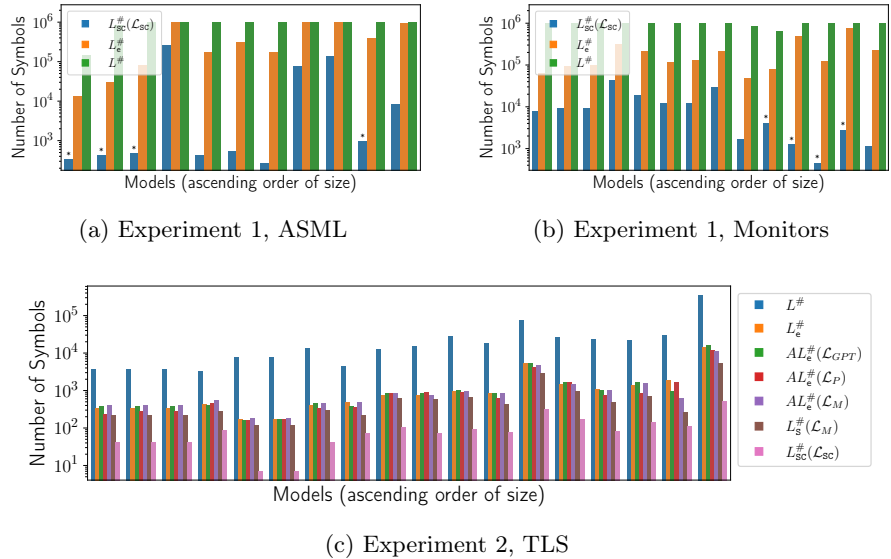

    \centering
    \begin{subfigure}[b]{0.49\textwidth}  
    \centering
    \resizebox{0.95\textwidth}{!}{
        \input{images/Experiment1_asml.pgf}}
    \subcaption{Experiment 1, ASML}
    \label{fig:exp1_asml}
    \end{subfigure}
    \hfill
    \begin{subfigure}[b]{0.49\textwidth}  
        \centering
        \resizebox{0.95\textwidth}{!}{
        \input{images/Experiment1_monitor.pgf}}
        \subcaption{Experiment 1, Monitors}
        \label{fig:exp1_mon}
    \end{subfigure}
    \vskip\baselineskip
    \begin{subfigure}[b]{0.99\textwidth}  
        \centering
        \resizebox{0.95\textwidth}{!}{
        \input{images/Experiment2.pgf}}
        \subcaption{Experiment 2, TLS}
        \label{fig:exp2} 
    \end{subfigure}
    \caption{Number of symbols (lower is better) for Experiment~1 and 2.}
    \label{fig:exp1}
\end{figure}

\myparagraph{Results.} \Cref{fig:exp1_asml,,fig:exp1_mon} show the number of symbols (log scale) for learning and testing on the y-axis for the ASML and monitor benchmark, respectively. The results for TLS are included in \Cref{fig:exp2} as the blue, orange and pink bars. The x-axis shows the models in ascending order of size $|\Sys| \times |I|$. The colors indicate the algorithms. Bars marked with * indicate the first hypothesis was correct.

\myparagraph{Discussion.} In every benchmark where $\erlsharp$ remains within the budget, it achieves better performance than $\lsharp$. Additionally, we observe that algorithm $\sclsharp$ significantly outperforms both $\lsharp$ and $\erlsharp$. 
For TLS (\Cref{fig:exp2}), the sound and complete references are sufficiently informative to learn all states without EQs (as indicated by the *). For the ASML and monitor benchmark, even when EQs are required, $\sclsharp$ requires far fewer symbols. For example, to the 115 states of the largest ASML model, $\sclsharp$ requires 8368 symbols on average, while $\erlsharp$ and $\lsharp$ exceed the budget of $10^6$ after learning only 73 and 27 states, respectively.

\subsection{RQ2: Using easy to obtain references}
Sound and complete reference models may not always be available. We investigate the performance on sound and arbitrary references on the TLS benchmark. 
We generate three types of references:
\begin{enumerate}[noitemsep,topsep=0pt]
    \item For each benchmark instance, we obtain a passively learned Mealy machine using the \emph{generalized state merging} algorithm with a pre-generated log file containing 10000 randomly generated traces. We then transform the Mealy machine into a minimal DFA with reference $\K_{P}$.
    \item We prompt chatGPT-5 to use RFC 5246, the open-source specification of TLS1.2, to construct four DFAs with reference $\K_{GPT}$ using the alphabets from~\cite{DBLP:conf/uss/RuiterP15}: server-side with a full alphabet or reduced alphabet, client-side with a full alphabet or reduced alphabet.
    \item We construct DFAs with reference $\K_M$ for the four specification cases by manually analyzing RFC 5246, leading to sound expert-based references.
\end{enumerate}
We compare $\eralsharp(\K)$ for $\K \in \{\K_P, \K_{GPT}, \K_M\}$, $\soundlsharp(\K_M)$ and $\erlsharp$.

\myparagraph{Results.} \Cref{fig:exp2} shows the total number of symbols (log scale) as the height of the bar and the TLS models in ascending order of size on the x-axis.

\myparagraph{Discussion.} First, we observe that $\soundlsharp$ always outperforms the $\erlsharp$ and $\eralsharp$, thus  performance improves whenever sound references are available. When comparing $\erlsharp$ and $\eralsharp$, we notice that the performance is usually quite similar. 
This may be explained by the fact that all models are relatively easy to learn or the references are too generic or not informative enough. Similar patterns were observed on limited tests with passively generated references for the ASML benchmark.

\subsection{RQ3: Violated reference assumptions}
We use the TLS benchmark to investigate whether the correct Mealy machine can be learned  when the references violates the soundness or completeness assumption.
We use the references $\K_{P}$ and $\K_{M}$ from RQ2, $\K_{\sound\complete}$ from RQ1. The reference $\K_{MUT}$ is a mutation on $\K_{\sound\complete}$ that redirects four transitions in the minimal DFA representation. For each run, we measure whether (1)~the algorithm terminated with the correct model, (2)~the algorithm terminated with an incorrect model, or (3)~the algorithm detected that the assumptions were violated.

\begin{table}[t]
    \caption{Results RQ3, values represent percentages over all seeds.}
    \label{tab:exp3}
    \centering
    \scalebox{0.8}{
        \begin{tabular}{|l|ccc|ccc|ccc|}
            \toprule
            Algorithm $\to$ & $\eralsharp(\K)$ & $\soundlsharp(\K)$ & $\sclsharp(\K)$ \\
            Reference $\downarrow$ & Incorrect(Violations) & Incorrect(Violations) & Incorrect(Violations) \\
            \midrule
            $\K_{MUT}$ & \cellcolor{green!25}0 & \cellcolor{red!25}55.6(0) & \cellcolor{red!25}100(100) \\
            $\K_{P}$   & \cellcolor{green!25}0 & \cellcolor{red!25}87.1(0) & \cellcolor{red!25}95.1(85.3) \\
            $\K_{M}$ & \cellcolor{green!25}0 & \cellcolor{green!25}0 & \cellcolor{red!25}100(100) \\
            $\K_{\sound\complete}$ & \cellcolor{green!25}0 & \cellcolor{green!25}0 & \cellcolor{green!25}0 \\
            \bottomrule
            \end{tabular}
    }
    \end{table}

\myparagraph{Results.} \Cref{tab:exp3} indicates per algorithm and reference combination how often a model was correctly learned, incorrectly learned or the completeness assumption was violated. The value is the percentage over all seeds and all TLS models.

\myparagraph{Discussion.} As expected, when a reference satisfies the assumption, the correct model is always learned. If a reference is incomplete, this is typically detected during the $\sclsharp$ algorithm, which reports an assumption violation. However, from $\K_{MUT}$ and $\K_{P}$ we observe that when references fail to meet the soundness assumption, the learned model is usually incorrect. This demonstrates that $\soundlsharp$ and $\sclsharp$ should only be applied when the assumption truly holds.

\section{Related Work} \label{sec:related}
Error-persistence and knowledge about inputs that lead to errors can be broadly categorized into using domain knowledge in AAL.

\myparagraph{Domain knowledge in learning.} 
Closest to our work is \emph{adaptive AAL}~\cite{DBLP:conf/ifm/DamascenoMS19,DBLP:conf/birthday/FerreiraHS22,DBLP:journals/igpl/GrocePY06,DBLP:conf/fm/KrugerJR24}. 
Adaptive AAL tries to reduce the number of interactions required to learn the SUL by using a similar reference model. 
It assumes that the reference model is of the same type as the SUL. In contrast, this paper uses a reference language consisting of words that do not produce errors. We thus rely on less information as our references provides no details about non-error outputs. We believe that our references are easier to obtain from, e.g., specifications.
Further, contrary to~\cite{DBLP:conf/ifm/DamascenoMS19,DBLP:conf/birthday/FerreiraHS22,DBLP:journals/igpl/GrocePY06,DBLP:conf/fm/KrugerJR24}, we support additional assumptions on the relation between the reference and SUL. These assumptions, while realistic, prove to be very powerful.

Maler and Mens~\cite{DBLP:conf/tacas/MalerM14} tackle AAL with large alphabets by learning symbolic automata whose transitions are defined by predicates that group individual inputs.
The correctness of their approach depends on the ability to find adequate predicates and finding such predicates is an open challenge. In contrast, we rely on knowledge about errors and prove correctness for varying levels of knowledge.

Wallner \etal~\cite{DBLP:conf/icst/WallnerALT25} learn timed automata using an untimed skeleton of the target model, which may be incomplete or incorrect. Similar to our approach, they assume that domain knowledge in the form of a skeleton is often easy to obtain. Instead of embedding their skeleton within the MAT framework, they use the skeleton within a genetic programming approach to add timing information.

Dengler \etal~\cite{DBLP:journals/corr/abs-2508-16384} learn Mealy machines with stochastic delays by first deriving a behavioral Mealy machine and then extending it with distributions that represent probabilistic delay behavior. Their contribution lies in extending Mealy machines to capture stochastic delay behavior, whereas our approach uses error-aware domain knowledge to accelerate the learning of the (standard) Mealy machine.

Henry \etal~\cite{DBLP:conf/concur/Henry0NS25} address the problem of large alphabets by assuming a concurrent component-based system in which certain inputs are disabled for specific components and can therefore be omitted. In contrast, we make no assumptions about component structure and operate at a different granularity, using the fact that certain input words (i.e., transitions) are not enabled.

\myparagraph{Domain knowledge in testing.} Yang \etal\cite{DBLP:conf/wcre/YangASLHCS19} test hypotheses on log traces before conducting the actual EQ when learning the ASML benchmark. 
To learn a model of embedded control software of a printer, related inputs were grouped into subalphabets to generate more effective test words by Smeenk \etal\cite{DBLP:conf/icfem/SmeenkMVJ15}. 
These two types of domain knowledge can be transformed into a reference language that our approach can use to avoid posing test queries that likely produce errors. 

Yaacov \etal~\cite{DBLP:journals/corr/Yaacov} use error persistence to adapt the EQ to learn informative bug descriptions. Their approach focuses on learning small DFAs that capture bugs, whereas we aim to learn the overall behavior of the SUL.
Ruiter \etal~\cite{DBLP:conf/uss/RuiterP15} cuts off tests when \emph{ConnectionClosed} is observed.
Our adjustments to test suites build on these two methods but are more general, applying to any error-related outputs and enabling test truncation before an error-producing input is reached through domain knowledge.

Rather than relying on domain knowledge, Kruger \etal\cite{DBLP:conf/tacas/KrugerJR24} infer assumptions about the SUL from the current hypothesis and use them to prioritize tests. 
Since tests are only prioritized but never discarded, the correct model is always learned. In contrast, when our assumptions hold, we can avoid posing certain queries altogether, leading to significantly fewer interactions.

Finally, Vaandrager and Melse~\cite{DBLP:conf/concur/VaandragerM25} introduce $k$-$A$-complete fault domains where $A$ typically represents the state cover of the specification of the SUL. If $A$ is a state cover of some sound and complete reference $\K$, then $\Sys$ is guaranteed to be in the $k$-$A$-complete fault domain, if $k$ is sufficiently large (see \Cref{lem:k-A-proof}).

\myparagraph{Special outputs.} Several AAL algorithms consider teachers that answer \emph{don't care} to an input when parts of the SUL are underspecified, unobservable, or irrelevant~\cite{DBLP:conf/ecoop/MoellerWSKF023,DBLP:conf/isola/LeuckerN12,DBLP:conf/tacas/BorkCGKM24,DBLP:journals/corr/Yaacov}. In some of these algorithms, a \emph{don't care} may be replaced with any output that is consistent with the data. This approach leads to smaller models compared to using \er when inputs are unobservable or irrelevant but the decision problem underlying finding a minimal model in the presence of don't-cares is NP-hard~\cite{DBLP:conf/ecoop/MoellerWSKF023,DBLP:conf/isola/LeuckerN12}.

\section{Conclusion}
This paper presents learning and conformance testing techniques that effectively incorporate knowledge about error sources in systems. Our algorithms improve over standard AAL with several orders of magnitude by assuming the system is error persistent. Using knowledge about \emph{which} inputs yield errors substantially lowers the cost of learning complex systems even further.

\myparagraph{Future work.}
This paper goes beyond the standard MAT framework by using a more fine-grained notion of equivalence that incorporates domain knowledge. This affects both the theory of testing and learning and required significant technical work. We therefore suggest to investigate a fundamental MAT framework for flexible notions of equivalence, thereby accommodating, e.g., more general domain knowledge.
Further, we aim to investigate how we can broaden our notion of errors to include transient errors that quickly resolve themselves or recoverable errors that allow escape from an error state after acknowledging, and to examine whether a more fine-grained definition of errors can account for multiple sink states.
On the practical side, we aim to investigate more applications and sources for error-aware references, such as sound references extracted from logs. 
Finally, we plan to compare our results, in terms of query complexity and runtime, with recent work using `don't cares' instead of $\er$-outputs~\cite{DBLP:conf/ecoop/MoellerWSKF023,DBLP:conf/isola/LeuckerN12,DBLP:conf/tacas/BorkCGKM24}.

\clearpage
\bibliographystyle{splncs04}
\bibliography{references.bib}
\camerareadyversion{}{
\appendix
\section{Algorithms} \label{app:algorithms}
We detail the full $\erlsharp$, $\eralsharp$ and $\sclsharp$ algorithms in this section. 
Additionally, we specify how $\textsc{OQ}_\er$ and $\textsc{OQ}_{\sound,\er}$ are defined procedurally. For this, we assume there are functions $\textsc{Step}$ and $\textsc{Reset}$ that execute an input on the SUL or reset it.
\begin{algorithm}[h]
	\begin{algorithmic}
		\Procedure{$OQ_{\er}$}{$\sigma \in I^*$}
		\State outputs $\gets \varepsilon$
      \State $\textsc{Reset}()$
      \For{$i \in \sigma$}
         \State outputs $\gets$ outputs $\cdot~ \textsc{Step}(i)$
         \StateIf{$\mathsf{last}($outputs$) = \er$}{ \textbf{break}}
      \EndFor
      \State \Return outputs
		\EndProcedure
	\end{algorithmic}
	\caption{Output queries with persistent $\er$}
	\label{alg:oq_er}
\end{algorithm}

\begin{algorithm}[h]
	\begin{algorithmic}
		\Procedure{$OQ_{\er,\sound}$}{$\sigma \in I^*,\K$}
		\State outputs $\gets \varepsilon$
      \State $r \gets$ initial state of the minimal DFA representing $\K$
      \State $\textsc{Reset}()$
      \For{$i \in \sigma$}
         \StateIf{$\delta^{\K}(r,i) \notin Q^{\K}_F$}{ \textbf{break}}
         \State outputs $\gets$ outputs $\cdot~ \textsc{Step}(i)$
         \StateIf{$\mathsf{last}($outputs$) = \er$}{ \textbf{break}}
      \EndFor
      \State \Return outputs
		\EndProcedure
	\end{algorithmic}
	\caption{Output queries with persistent $\er$ and sound reference $\K$}
	\label{alg:oq_er_sound}
\end{algorithm}

Next, we list \textsc{ProcCounterEx} originated in \cite{DBLP:conf/tacas/VaandragerGRW22} but re-iterated in in \Cref{alg:counterexamplebs} for completeness. Note that we use variants where \emph{only} the definition of apartness has been adapted, as indicated by comments.

\begin{algorithm}[h]%
    \begin{algorithmic}
       \Procedure{\ProcessCounterexample}{$\Hyp$, $\sigma \in I^*$}
       \State $q \gets \delta^{\Hyp}(q_0^{\Hyp}, \sigma)$
       \State $r \gets \delta^{\Obs}(q_0^{\Obs}, \sigma)$
       \If{$r \in S \cup F$}
       \State \Return
       \Else
       \State $\rho \gets$ unique prefix of $\sigma$ with $\delta^{\Obs}(q_0^{\Obs}, \rho) \in  F$
       \State $h \gets \lfloor \frac{| \rho | +| \sigma |}{2} \rfloor$
       \State $\sigma_1 \gets \sigma[1..h]$
       \State $\sigma_2 \gets \sigma[h+1 .. | \sigma |]$
       \State $q' \gets \delta^{\Hyp}(q_0^{\Hyp}, \sigma_1)$
       \State $r' \gets \delta^{\Obs}(q_0^{\Obs}, \sigma_1)$
       \State $\eta \gets$ witness for $q \;\#\; r$ \Comment{or $\apart_{\sound,\K}$, $\apart_{\sound\complete,\K}$}
       \State $\OutputQuery(\mathsf{access}(q') \; \sigma_2 \; \eta)$
       \If{$q' \;\#\; r'$} \Comment{or $\apart_{\sound,\K}$, $\apart_{\sound\complete,\K}$}
       \State \ProcessCounterexample($\Hyp$, $\sigma_1$)
       \Else
       \State \ProcessCounterexample($\Hyp$, $\mathsf{access}(q') \; \sigma_2$)
       \EndIf
       \EndIf
       \EndProcedure
    \end{algorithmic}
   \caption{Processing of $\sigma$
     that leads to a conflict.
   }
   \label{alg:counterexamplebs}
 \end{algorithm}%

The next three algorithms are variations of $\lsharp$ described in \Cref{sec:learning}. Specifically, $\erlsharp$ (\Cref{sec:lsharp_perm}) is listed in \Cref{alg:persistent_errors_lsharp}, $\eralsharp$ (\Cref{sec:adap_lsharp_perm}) is listed in \Cref{alg:persistent_errors_alsharp} and $\sclsharp$ is listed in (\Cref{sec:sound_compl_alg}).
Note that some additional pre-conditions for match separation are listed in \Cref{alg:persistent_errors_alsharp}, these pre-conditions ensure progress.

\begin{algorithm}[h]
	\begin{algorithmic}
		\Procedure{LSharp$_{\er}$}{}
		\DoIf{$q$ isolated, for some $q \in \erfrontier$}
      \Comment{promotion}
      \State $\erbasis \gets \erbasis \cup \{ q \}$
		\ElsDoIf{$\delta^{\Obs}(q,i){\uparrow}$, for some $q \in \basis, i \in I$}
      \Comment{extension}
      \State $\textsc{OQ}_\er(\mathsf{access}(q) \; i)$
		\ElsDoIf{$\neg(q\apart r)$, $\neg(q\apart r')$,$r\neq r'$ for some $q \in \erfrontier$ and $r,r'\in \erbasis$}
      \Comment{separation}
      \State $\sigma \gets \text{witness of \(r\apart r'\)}$
      \State $\textsc{OQ}_\er(\mathsf{access}(q) \; \sigma)$
		\ElsDoIf{$\erbasis$ is adequate}
      \Comment{equivalence}
		  \State $\Hyp \gets \BuildHypothesis_{\er}$ 
		  \State $(b, \sigma) \gets \CheckConsistency(\Hyp)$		
		  \If{$b = \texttt{yes}$}
         \State $(b, \rho) \gets \EquivalenceQuery(\Hyp)$
         \If{$b = \texttt{yes}$} 
            \State \Return $\Hyp$
         \Else
            \State $\textsc{OQ}_\er(\rho)$
            \State $\sigma \gets$ shortest prefix of $\rho$ such that
            $\delta^{\Hyp}(q_0^\Hyp, \sigma) \apart \delta^{\Obs}(q_0^\Obs, \sigma)$
         (in $\Obs$)
         \EndIf
        \EndIf
		  \State $\textsc{ProcCounterEx}_\er(\Hyp, \sigma)$ 
        \Comment{Uses $\textsc{OQ}_\er$}
    \EndDoIf
		\EndProcedure
	\end{algorithmic}
	\caption{$\lsharp$ with persistent $\er$ (see \Cref{sec:lsharp_perm}).}
	\label{alg:persistent_errors_lsharp}
\end{algorithm}

\begin{algorithm}[h]
	\begin{algorithmic}
		\Procedure{L$^{\#}_{\er,\K}$}{$\K$}
      \State $\R_\K \gets$ minimal DFA representation of $\K$
      \State $P, W \gets$ minimal state cover of $\R_\K$, separating family of $\R_\K$ 
   
      \DoIf{$(\delta^{\Obs}(q,i) \in \erfrontier \lor \delta^{\Obs}(q,i){\uparrow}), \neg(q' \apart \delta^{\Obs}(q,i)),$ $\access^{\Obs}(q)i, \access^{\Obs}(q') \in P$,\\
      \qquad\qquad$\sigma = \mathsf{sep}(W,\delta^{\R_\K}(\access^{\Obs}(q)i),\delta^{\R_\K}(\access^{\Obs}q'))$ and $\sigma$ guarantees progress\\
      \qquad\qquad for some $q, q' \in \erbasis, i \in I$}
      \Comment{rebuilding}
      \State $\textsc{OQ}_\er(\access^{\Obs}(q)i\sigma)$
      \State $\textsc{OQ}_\er(\access^{\Obs}(q')\sigma)$
      \ElsDoIf{$q$ isolated, for some $q \in \erfrontier$ with $\access^{\Obs}(q) \in P$}
      \Comment{prioritized promotion}
      \State $\erbasis \gets \erbasis \cup \{ q \}$
      \EndDoIf
		\DoIf{$q$ isolated, for some $q \in \erfrontier$}
      \Comment{promotion}
      \State $\erbasis \gets \erbasis \cup \{ q \}$
		\ElsDoIf{$\delta^{\Obs}(q,i){\uparrow}$, for some $q \in \basis, i \in I$}
      \Comment{extension}
      \State $\textsc{OQ}_\er(\mathsf{access}(q) \; i)$
		\ElsDoIf{$\neg(q\apart r)$, $\neg(q\apart r')$, $r\neq r'$ for some $q \in \erfrontier$ and $r,r'\in \erbasis$}
      \Comment{separation}
      \State $\sigma \gets \text{witness of \(r\apart r'\)}$
      \State $\textsc{OQ}_\er(\mathsf{access}(q) \; \sigma)$

      \ElsDoIf{$\delta^{\Obs}(q,i)=r, \delta^{\K}(p,i)=p', p \matches q,$\\
      \qquad\qquad$\neg(\exists \rho \in I^*. \mathsf{last}(\lambda^\Obs(q,\rho))=\er \leftrightarrow \delta^{\R_\K}(p,\rho) \in Q^{\R_\K}_F)$,\\
      \qquad\qquad and $\neg\exists s \in \erbasis$ with $s \matches p'$, $\neg(q' \apart r)$\\
      \qquad\qquad $\exists \sigma \in I^*$ s.t. $\last(\lambda^\Obs(q',\sigma))=\er \leftrightarrow \delta^{\R_\K}(p',\sigma)\in Q^{\R_\K}_F$\\
      \qquad\qquad $q, q' \in \erbasis, r \in \erfrontier, i \in I, p, p' \in Q^{\R_\K}$}
      \Comment{match separation}
      \State choose $\sigma$ s.t. $\last(\lambda^\Obs(q',\sigma))=\er \leftrightarrow \delta^{\R_\K}(p',\sigma)\in Q^{\R_\K}_F$
      \State $\textsc{OQ}_\er(\mathsf{access}(q) \; i\sigma)$

      \ElsDoIf{$p \matches q$ and $p' \matches q$ for some $q \in B$ and $p, p' \in Q^{\R_\K}$ with\\
      \qquad\qquad$\sigma = \mathsf{sep}(W,p,p')$ and $\delta^\Obs(q,\sigma){\uparrow}$}
      \Comment{match refinement}
      \State $\textsc{OQ}_\er(\mathsf{access}(q) \; \sigma)$

		\ElsDoIf{$\erbasis$ is adequate}
      \Comment{equivalence}
		  \State $\Hyp \gets \BuildHypothesis_{\er}$ 
		  \State $(b, \sigma) \gets \CheckConsistency(\Hyp)$		
		  \If{$b = \texttt{yes}$}
         \State $(b, \rho) \gets \EquivalenceQuery(\Hyp)$
         \If{$b = \texttt{yes}$} 
            \State \Return $\Hyp$
         \Else
            \State $\textsc{OQ}_\er(\rho)$
            \State $\sigma \gets$ shortest prefix of $\rho$ such that
            $\delta^{\Hyp}(q_0^\Hyp, \sigma) \apart \delta^{\Obs}(q_0^\Obs, \sigma)$
         (in $\Obs$)
         \EndIf
        \EndIf
		  \State $\textsc{ProcCounterEx}_\er(\Hyp, \sigma)$
    \EndDoIf
		\EndProcedure
	\end{algorithmic}
	\caption{$\alsharp$ with persistent $\er$ (see \Cref{sec:adap_lsharp_perm}).}
	\label{alg:persistent_errors_alsharp}
\end{algorithm}

\begin{algorithm}[h]
	\begin{algorithmic}
		\Procedure{L$^{\#}_{\er,\sound\complete}$}{$\K$}
      \State $P \gets$ minimal state cover of DFA representation of $\K$ 
      \State Execute $T_\K$ 
      \Comment{\Cref{eq:t}}
      \State $\erbasis \gets \{ \delta^\Obs(p) \mid p \in P \}$
      \DoIf{$\lambda^{\Obs}(\sigma)=\er, \delta^{\K}(\sigma) \in Q^{\K}_F$, for some $\sigma \in I^*$}
      \Comment{completeness violation}
      \State \Return \texttt{error}
		\ElsDoIf{$q$ isolated, for some $q \in \erfrontier$}
      \Comment{promotion}
      \State $\erbasis \gets \erbasis \cup \{ q \}$
		\ElsDoIf{$\delta^{\Obs}(q,i){\uparrow}, \access^{\Obs}(q)i\in \K$, for some $q \in \basis, i \in I$}
      \Comment{extension}
      \State $\textsc{OQ}_{\sound}(\mathsf{access}(q) \; i, \K)$
		\ElsDoIf{$\neg(q\apart_{\sound\complete,\K} r)$, $\neg(q\apart_{\sound\complete,\K} r')$, $r\neq r'$\\
      \qquad\qquad for some $q \in \erfrontier$ and $r,r'\in \erbasis$}
      \Comment{separation}
      \State $\sigma \gets \text{witness of \(r\apart_{\sound\complete,\K} r'\)}$
      \State $\textsc{OQ}_{\sound}(\mathsf{access}(q) \; \sigma,\K)$

		\ElsDoIf{$\erbasis$ is $\sound$-adequate}
      \Comment{equivalence}
		  \State $\Hyp \gets \BuildHypothesis_{\er}$ 
		  \State $(b, \sigma) \gets \CheckConsistency(\Hyp)$	
        \If{$b = \texttt{yes}$}
         \State $(b, \rho) \gets \CheckSoundness(\K,\Hyp)$	
		  \If{$b = \texttt{yes}$}
         \State $(b, \rho) \gets \EquivalenceQuery(\Hyp)$
         \Comment{equivalence on $\K$}
         \If{$b = \texttt{yes}$} 
            \State \Return $\Hyp$
         \EndIf
         \EndIf
            \State $\textsc{OQ}_{\sound}(\rho,\K)$
            \State $\sigma \gets$ shortest prefix of $\rho$ such that
            $\delta^{\Hyp}(q_0^\Hyp, \sigma) \apart_{\sound\complete,\K} \delta^{\Obs}(q_0^\Obs, \sigma)$
         (in $\Obs$)
         \EndIf
		  \State $\textsc{ProcCounterEx}_{\sound\complete}(\Hyp, \sigma)$
    \EndDoIf
		\EndProcedure
	\end{algorithmic}
	\caption{$\erlsharp$ for a sound and complete reference (see \Cref{sec:sound_compl_alg}).}
	\label{alg:sound_complete_lsharp}
\end{algorithm}

\clearpage 
\section{Proofs}
\subsection{Proofs of Section 3}
\subsubsection{Proof of \Cref{lem:error_test_suite}}
  \textit{(Proof by contradiction)} Assume $T_\er(\Hyp)$ is not $\er$-complete but $T$ is complete for some fault domain $\C$. 
  From the assumption that $T_\er(\Hyp)$ is not $\er$-complete, we derive that there exists a Mealy machine $\Sys$ such that
  \begin{enumerate}
      \item $\Hyp \nsim \Sys$ but $\Hyp \sim_{T_\er(\Hyp)} \Sys$ (from \Cref{def:compl}),
      \item $\Sys$ is an $\er$-persistent Mealy machine (\Cref{def:kaercompl}),
      \item $\Sys \in \C$ (\Cref{def:kaercompl}).
  \end{enumerate}
  From (1), (3) and the assumption $T$ is complete, we derive that there must be some $\sigma \in I^*$ such that (4) $\lambda^{\Sys}(\sigma) \neq \lambda^{\Hyp}(\sigma)$, (5) $\sigma \in T$ and (6) $\sigma \notin T_\er(\Hyp)$.

  From (5), (6) and \Cref{def:ertest}, it must be the case that $\er\er$ is a subword of $\lambda^{\Hyp}(\sigma)$. Let $i \in \mathbb{N}$ be such that $\firsterror{\Hyp}{\sigma} < \infty$, i.e., $\lambda^{\Hyp}(\sigma)_i \neq \er$ and $\lambda^{\Hyp}(\sigma)_{i+1} = \er$.
  Because $\sigma_{:i+1}$ is contained in $T_\er(\Hyp)$, if $\lambda^{\Hyp}(\sigma)_{:i+1} \neq \lambda^{\Sys}(\sigma)_{:i+1}$ then this would contradict the assumption that $T_\er(\Hyp)$ is not $\er$-complete. Thus, we can safely assume $\lambda^{\Hyp}(\sigma)_{:i+1} = \lambda^{\Sys}(\sigma)_{:i+1}$.

  We perform a case distinction based on $\lambda^{\Sys}(\sigma)_{i+1}$:
  \begin{itemize}
      \item Case $\lambda^{\Sys}(\sigma)_{i+1} = \er$. Because $\Hyp$ and $\Sys$ are $\er$-persistent, for all $i+1 < j < |\sigma|$, $\lambda^{\Sys}(\sigma)_j = \er = \lambda^{\Hyp}(\sigma)_j$. This contradicts with (4).
      
      \item Case $\lambda^{\Sys}(\sigma)_{i+1} \neq \er$. From \Cref{def:ertest}, it holds that $\sigma_{:i+1} \in T_\er(\Hyp)$. Because $\lambda^{\Sys}(\sigma)_{i+1} \neq \er$, $\lambda^{\Sys}(\sigma)_{i+1} \neq \er = \lambda^{\Hyp}(\sigma)_{i+1}$. Thus, there is a counterexample $\sigma_{:i+1} \in  T_\er(\Hyp)$ that shows $\Sys \nsim \Hyp$ which contradicts with the assumption that $T_\er(\Hyp)$ is not $\er$-complete.
  \end{itemize}
  Because all cases lead to a contradiction, $T_\er(\Hyp)$ must be $\er$-complete. \qed

\subsubsection{Proof of \Cref{lem:sound_test_suite}}~\newline
  \textit{(Proof by contradiction)} Assume $T_\sound(\K,\Hyp)$ is not $\er$-complete but $T$ is complete for some fault domain $\C$.
  From the assumption that $T_\sound(\K,\Hyp)$ is not $\er$-complete, we derive that there exists a Mealy machine $\Sys$ such that
  \begin{enumerate}
      \item $\Hyp \nsim \Sys$ but $\Hyp \sim_{T_\sound(\K,\Hyp)} \Sys$ (from \Cref{def:compl}),
      \item $\Sys$ is an $\er$-persistent Mealy machine (\Cref{def:kaercompl}),
      \item $\Sys \in \C$ (\Cref{def:kaercompl}).
  \end{enumerate}
  From (1), (3) and the assumption $T$ is complete, we derive that there must be some $\sigma \in I^*$ such that (4) $\lambda^{\Sys}(\sigma) \neq \lambda^{\Hyp}(\sigma)$, (5) $\sigma \in T$ and (6) $\sigma \notin T_\sound(\K,\Hyp)$.

  From (5), (6) and \Cref{def:ertest}, it must be the case that either $\er\er$ is a subword of $\lambda^{\Hyp}(\sigma)$ (which leads to the same case distinction as used in the proof of \Cref{lem:error_test_suite}) or there exists $i \in \mathbb{N}$ such that $\firsterror{\K}{\sigma} < \infty$, i.e., $\sigma_{:i} \in \K$ and $\sigma_{:i+1} \notin \K$. This latter case truncates $\sigma$ to $\sigma_{:i}$. 
  
  Because $\sigma_{:i}$ is contained in $T_\sound(\Hyp)$, if $\lambda^{\Hyp}(\sigma)_{:i} \neq \lambda^{\Sys}(\sigma)_{:i}$ then this would contradict the assumption that $T_\sound(\Hyp)$ is not $\er$-complete. Thus, we can safely assume $\lambda^{\Hyp}(\sigma)_{:i} = \lambda^{\Sys}(\sigma)_{:i}$.
  
  Because $\K$ is sound for both $\Hyp$ and $\Sys$, $\lambda^{\Sys}(\sigma)_{i+1} = \er = \lambda^{\Sys}(\sigma)_{i+1}$. Moreover, because $\Hyp$ and $\Sys$ are $\er$-persistent, for all $i+1 < j < |\sigma|$, $\lambda^{\Sys}(\sigma)_j = \er = \lambda^{\Hyp}(\sigma)_j$. This contradicts with (4).
  Because all cases lead to a contradiction, $T_\sound(\K,\Hyp)$ must be $\er$-complete. \qed

\subsubsection{Proof of \Cref{lem:mostly_test_suite}} 
From \Cref{lem:error_test_suite}, it follows that under these settings $T_\er(\Hyp)$ is complete w.r.t. $\C_\er$. Because $\moe(T_\er(\Hyp),$ $T_\sound(\Hyp,\K))$ only terminates when all test suites have been executed or a counterexample has been found: 
\[ \Hyp \sim_{T_\er(\Hyp)} \Sys \implies \Hyp \sim_{\moe(T_\er(\Hyp),T_\sound(\Hyp,\K))} \Sys.\] 
Thus, $\moe(T_\er(\Hyp),$ $T_\sound(\Hyp,\K))$ is complete w.r.t. $\C_\er$.
\qed 

\subsection{Proofs for Section 4.1 and 4.2}
Before proving \Cref{lem:error_learning} and \Cref{lem:mostly_learning}, we prove three lemmas related to hypothesis building and counterexample processing.

\begin{lemma} \label{lem:hyp_er_persistent}
  Given an observation tree $\Obs$ with basis $\erbasis$ and frontier $\erfrontier$, procedure $\textsc{BuildHypothesis}_{\er}$ constructs an $\er$-persistent hypothesis $\Hyp$.
\end{lemma}
\begin{proof}
  By definition of frontier, if $r \in \erfrontier$ then there exists some $b \in \erbasis$ and $i \in I$ such that $\delta^{\Obs}(b,i)=r$ iff $\lambda^{\Obs}(b,i) \neq \er$. Additionally, only frontier states can be added to the basis according to the promotion rule. Therefore, for any transition $\lambda^\Obs(b, i) = \er$ with $b \in B$ and $i \in I$, must trigger the first case in \Cref{build_hyp} and lead to $\lambda^{\Hyp}(b,i)=\er$ and $\delta^{\Hyp}(b,i)=q_\er$. When $q_\er$ in the hypothesis is reached, all transitions self-loop with output $\er$, creating the sink state that characterizes $\er$-persistent systems. \qed
\end{proof}

\begin{lemma} \label{lem:proc_counter_er}
  Let $\Obs$ be an observation tree, $\Hyp$ an $\er$-persistent Mealy machine and $\sigma \in I^*$ such that $\delta^\Hyp(\sigma) \apart \delta^\Obs(\sigma)$. Then \textsc{ProcCounterEx}$_\er(\Hyp,\sigma)$ using $\textsc{OQ}_\er$ terminates and is correct, i.e., after termination $\Hyp$ is not a hypothesis for $\Obs$ anymore, and, in particular, a frontier state becomes isolated.
\end{lemma}
\begin{proof}
This proof combines Lemma 3.11 and Theorem 3.14 from~\cite{DBLP:conf/tacas/VaandragerGRW22}. Note that in our case, the equivalence rule can only be applied if $\erbasis$ is adequate, i.e., if the basis is complete and all frontier states are identified. The only difference between \textsc{ProcCounterEx}$_\er$$(\Hyp,\sigma)$ compared to the listing in~\cite{DBLP:conf/tacas/VaandragerGRW22} is the use of $\textsc{OQ}_\er$ instead of $\textsc{OQ}$. Termination trivially follows the proof of Lemma 3.11 when adding the base case $r \in \{ \delta^\Obs(b,i) \mid b \in \erbasis, i \in I \}$ but $r \notin \erbasis ~\cup~ \erfrontier$. It remains to prove that  \textsc{ProcCounterEx}$_\er$$(\Hyp,\sigma)$ leads to an isolated frontier state.
Let $q := \delta^\Hyp(\sigma), r:= \delta^\Obs(\sigma)$ and $\eta \vdash q \apart r$ for some $\eta \in I^*$, i.e., $\lambda^\Obs(q,\eta) \neq \lambda^\Obs(r,\eta)$.

  \begin{itemize}
    \item Suppose $r \in \{ \delta^\Obs(b,i) \mid b \in \erbasis, i \in I \}$ but $r \notin \erbasis~\cup~\erfrontier$. It follows that $\delta^\Obs(b,i)=r$ and $\lambda^\Obs(b,i)=\er$ for some $b \in B, i \in I$. This case cannot occur as $\textsc{OQ}_\er$ ensures that any trace is cut-off after observing the first $\er$. In particular, there can exist no $\eta \in I^*$ such that $\eta \vdash q \apart r$.
    
    \item Suppose $r \in \erbasis \cup \erfrontier$. This case is equivalent to the proof in Lemma 3.11 as it does not make use of $\textsc{OQ}_\er$. Note that this results in $q \apart r$ with $r \in \erfrontier$ and $q \in \erbasis$ and because $r$ was identified when constructing the hypothesis, $r$ must now be isolated. 
    
    We prove $q \neq q_\er$, with $q_\er$ as in \textsc{BuildHypothesis}$_\er$, by contradiction. 
    Let $\sigma = \alpha i$ with $\alpha \in I^*$ and $i \in I$. 
    Because $r \in \erbasis \cup \erfrontier$, there exists some $r' \in \erbasis$ with $\delta^\Obs(\alpha i) = \delta^\Obs(r',i)=r$. 
    Let $q' = \delta^\Hyp(\alpha)$.
    Because $r' \in \erbasis$, and $\Obs$ is a tree with $\alpha = \access(r')$, it must hold that $q' = r'$ and thus $q' \in \erbasis$.
    For $q = q_\er$ to be true, $\lambda^\Obs(q',i) = \lambda^\Obs(r',i) = \er$ according to \textsc{BuildHypothesis}$_\er$. However, $\textsc{OQ}_\er$ ensures that any trace is cut-off after observing the first $\er$, which means that there can be no $\eta \in I^*$ such that $\eta \vdash q \apart r$.
    
    \item Suppose $r \notin \erbasis \cup \erfrontier$. We decompose  $\sigma$ into $\sigma_1\sigma_2$ such that $q' := \delta^\Hyp(\sigma_1) \in \erbasis$ and $r' := \delta^\Obs(\sigma_1)$ as in \Cref{alg:counterexamplebs}. We perform $\textsc{OQ}_\er(\access(q')\sigma_2\eta)$.
    Suppose afterwards $\delta^\Obs(q',\sigma_2\eta){\downarrow}$, in that case correctness and termination follow from Theorem 3.11.
    Suppose $\delta^\Obs(q',\sigma_2\eta){\uparrow}$ and let $\mu$ be the strict prefix of $\sigma_2\eta$ such that $\delta^\Obs(q',\mu){\downarrow}$. It follows from the definition of $\textsc{OQ}_\er$ that $\last(\lambda^\Obs(q',\mu))=\er$. 

    Because $\eta \vdash q \apart r$ and $\delta^\Obs(\sigma_1\sigma_2) = \delta^\Obs(\sigma) = r$, it must hold that $\lambda^\Obs(\sigma_1\sigma_2\eta){\downarrow}$.
    As $\mu$ is a strict prefix of $\sigma_2\eta$ and $\textsc{OQ}_\er$ ensures that any trace is cut-off after observing the first $\er$, it must hold that $\last(\lambda^\Obs(\sigma_1\mu)) = \last(\lambda^\Obs(r',\mu))\neq \er$.

    Thus, $\last(\lambda^\Obs(q',\mu))=\er \neq \last(\lambda^\Obs(r',\mu))$. In particular, $\delta^\Hyp(\sigma_1) \apart \delta^\Obs(\sigma_1)$ which makes $\sigma_1$ a valid parameter for \textsc{ProcCounterEx}$_\er$ and, by induction, $\Hyp$ is not a valid hypothesis for $\Obs$ anymore after the recursive call.\qed
  \end{itemize}
\end{proof}

\begin{lemma} \label{lem:cex_after_oq_er}
  Let $\Sys$ be $\er$-persistent Mealy machine, $\Obs$ an observation tree for $\Sys$, $\Hyp$ a hypothesis build using $\textsc{BuildHypothesis}_\er(\Obs)$, $\rho \in I^*$ such that $\lambda^\Hyp(\rho) \neq \lambda^\Sys(\rho)$. After executing $\textsc{OQ}_\er(\rho)$, there exists some $\mu \in I^*$ such that $\lambda^\Hyp(\mu) \neq \lambda^\Obs(\mu)$.
\end{lemma}
\begin{proof}
  If $\er \notin \lambda^\Sys(\rho)$, $\lambda^{\Obs}(\rho){\downarrow}$ and the above statement holds trivially.
  Assume $\er \in \lambda^\Sys(\rho)$, then there must exist some $i \in \mathbb{N}$ such that $\lambda^\Sys(\rho)_i \neq \er$ and $\lambda^\Sys(\rho)_{i+1} = \er$. 
  
  Suppose $\lambda^\Hyp(\rho)_{:i+1} \neq \lambda^\Sys(\rho)_{:i+1}$. In that case, $\textsc{OQ}_\er(\rho)$ results in $\lambda^\Obs(\rho_{:i+1}){\downarrow}$ and $\lambda^\Hyp(\rho)_{:i+1} \neq \lambda^\Obs(\rho)_{:i+1}$, so we set $\mu = \rho_{:i+1}$. 

  Suppose $\lambda^\Hyp(\rho)_{:i+1} = \lambda^\Sys(\rho)_{:i+1}$. In that case, $\lambda^\Sys(\rho)_{i+1} = \er = \lambda^\Hyp(\rho)_{i+1}$. However, we assume $\Sys$ is $\er$-persistent and $\Hyp$ is $\er$-persistent by construction (\Cref{lem:hyp_er_persistent}). 
  Thus for any $j \in \mathbb{N}$ with $i < j \leq |\rho|$, $\lambda^\Sys(\rho)_{j} = \er = \lambda^\Hyp(\rho)_{j}$ which contradict the assumption that $\lambda^\Hyp(\rho) \neq \lambda^\Sys(\rho)$. \qed
\end{proof}

\subsubsection{Proof of \Cref{lem:error_learning}}
If SUL $\Sys$ is $\er$-persistent, then $\erlsharp$ learns $\Sys$ within $\bigO(kn^2 + n \log m)$ OQs and at most $n - 1$ EQs where $n = |\Sys|$, $k = |I|$ and $m$ is the length of the longest counterexample.
\begin{proof}
    We prove that $\erlsharp$ learns $\Sys$ within $\bigO(kn^2 + n \log m)$ OQs and at most $n - 1$ EQs by:
    \begin{enumerate}
      \item Proving every rule application in $\erlsharp$ increases the following norm $N_{\er}(\Obs)$.
      \begin{equation*}
      \resizebox{0.91\hsize}{!}{%
          $N_{\er}(\Obs) = \underbrace{\frac{|\erbasis|(|\erbasis| + 1)}{2}}_{N_Q(\Obs)} + |\underbrace{\vphantom{\frac{|\erbasis|}{2}}\{(q,i)\in \basis \times I\mid \delta^{\Obs}(q,i){\downarrow}\}}_{N_{\downarrow(\Obs)}}| + |\underbrace{\vphantom{\frac{|\erbasis|}{2}}\{(q,q') \in \erbasis \times \erfrontier |\ q \apart q'\}}_{N_{\apart}(\Obs)}|
          $}
      \end{equation*}
      \item Next, we prove that this norm is bounded by $|\Sys|$ and $|I|$ and derive the maximum number of rule applications.
      \item We combine all previous parts to derive the complexity.
    \end{enumerate}

    \myparagraph{(1) Every rule application in $\erlsharp$ increases norm $N_\er(\Obs)$.} Let $\erbasis,\erfrontier,\Obs$ indicate the values before and $\erbasis',\erfrontier',\Obs'$ denote the values after a rule application. 
    This proof is similar to proofs in~\cite{DBLP:conf/tacas/VaandragerGRW22,DBLP:conf/fm/KrugerJR24}  but is added for completeness.
    \begin{description}
    \item[Promotion] If $q \in \erfrontier$ is isolated, we move it from $\erfrontier$ to $\erbasis$, i.e.~$\erbasis' := \erbasis\cup\{q\}$, thus
      \begin{align*}
        N_Q(\Obs') &= \frac{|\erbasis'| \cdot (|\erbasis'|+1)}{2}
        = \frac{(|\erbasis|+1) \cdot (|\erbasis|+1+1)}{2}
        \\
                   &= \frac{(|\erbasis|+1)\cdot |\erbasis|}{2} + \frac{(|\erbasis|+1)\cdot 2}{2}
                     = N_Q(\Obs) + |\erbasis| + 1 
                     \\
        N_\downarrow(\Obs') &\supseteq N_\downarrow(\Obs)
        \\
        N_{\apart}(\Obs') &\supseteq N_{\apart}(\Obs) \setminus (\erbasis\times \{q\})
      \end{align*}
      From which it follows that $|N_{\apart}(\Obs')| ~~\ge~~ |N_{\apart}(\Obs)| - |\erbasis|$. In total, $N_\er(\Obs') \ge N_\er(\Obs)+1$.
  
    \item[Extension] If $\delta^{\Obs}(q,i){\uparrow}$ for some $q\in \basis$ and $i\in I$, we execute $\textsc{OQ}_\er(\mathsf{access}(q) \, i)$. By definition of the frontier and basis, $\er \notin \lambda^{\Obs}(\access(q))$ for all $q \in \erbasis$. Thus, this may result in either $\lambda^{\Obs}(q,i) = \er$ or $\lambda^{\Obs}(q,i) \neq \er$ but in both cases,
      \[
        N_Q(\Obs') = N_Q(\Obs)
        \qquad
        N_\downarrow(\Obs') = N_{\downarrow}(\Obs)\cup\{(q,i)\}
        \qquad
        N_{\apart}(\Obs') \subseteq N_{\apart}(\Obs)
      \]
      and thus $N_\er(\Obs') \ge N_\er(\Obs) + 1$.
    \item[Separation] Let $q\in \erfrontier$ and distinct $r,r'\in \erbasis$
      with $\neg(q\apart r)$ and $\neg(q\apart r')$. The algorithm performs 
      the query $\textsc{OQ}_\er(\mathsf{access}(q) \; \sigma)$. 
      We perform a case distinction based on whether $\lambda^{\Obs}(q,\sigma){\downarrow}$ after executing the $\textsc{OQ}_\er(\mathsf{access}(q) \; \sigma)$:
      \begin{itemize}
        \item $\lambda^{\Obs}(q,\sigma){\downarrow}$ which implies $r\apart q$ or $r' \apart q$.
        \[
        N_{\apart}(\Obs') \supseteq N_{\apart}(\Obs) \cup \{(r,q)\}
        \qquad\text{ or } \qquad 
        N_{\apart}(\Obs') \supseteq N_{\apart}(\Obs) \cup \{(r',q)\}
        \]
        and therefore $|N_{\apart}(\Obs')| \ge |N_{\apart}(\Obs)| + 1$.
        The other components of the norm stay unchanged, thus the norm rises.
        \item $\lambda^{\Obs}(q,\mu){\downarrow}$ for $\mu$ a strict prefix of $\sigma$ and $\last(\lambda^{\Obs}(q,\mu)) = \er$. As $\textsc{OQ}_\er$ prevents adding any word to $\Obs$ with more than one $\er$ output, we can safely assume that for any strict prefix $\rho$ of $\sigma$, $\last(\lambda^{\Obs}(r,\rho)) \neq \er$ and $\last(\lambda^{\Obs}(r',\rho)) \neq \er$. We derive $\last(\lambda^{\Obs}(r,\mu)) \neq \last(\lambda^{\Obs}(q,\mu)) \neq \last(\lambda^{\Obs}(r',\mu))$. Thus,
        \[
        N_{\apart}(\Obs') \supseteq N_{\apart}(\Obs) \cup \{(r,q), (r',q)\}
        \]
        and therefore $|N_{\apart}(\Obs')| \ge |N_{\apart}(\Obs)| + 2$.
        The other components of the norm stay unchanged, thus the norm rises.
      \end{itemize}
    \item[Equivalence] If the algorithm does not terminate after the equivalence rule, we show that the norm increases because a frontier state becomes isolated.
    By \textsc{EQ}, it holds that $\rho \in I^*$ such that $\lambda^\Hyp(\rho) \neq \lambda^\Sys(\rho)$. Contrary to $\lsharp$, we do not add $\rho$ to $\Obs$ but the result of $\textsc{OQ}_\er(\rho)$. 

    From \Cref{lem:cex_after_oq_er}, it follows that if the algorithm does not terminate after the equivalence query, then there exists some $\mu \in I^*$ such that $\lambda^\Hyp(\mu) \neq \lambda^\Obs(\mu)$. Thus, calling $\textsc{ProcCounterEx}_\er$ with $\sigma$ the shortest prefix of $\rho$ such that $\delta^\Hyp(\sigma) \apart \delta^\Obs(\sigma)$ results in an isolated frontier states (\Cref{lem:proc_counter_er}). In particular, this implies 
    \[
        (q,r) \in N_{\apart}(\Obs') \setminus N_{\apart}(\Obs).
        \tag*{\qed}
    \]
    As all other components of the norm remain the same, $N_\er(\Obs') \geq N_\er(\Obs) + 1$.
    \end{description}

    \myparagraph{(2) Norm $N_\er(\Obs)$ is bounded by $|\Sys|$ and $|I|$.}
    It must hold that $|\erbasis| \leq |\Sys|$. The frontier is bounded by $k \cdot n$, $|\{(q,i)\in \basis \times I\mid \delta^{\Obs}(q,i){\downarrow}\}| \leq kn$. The set $\erbasis \cup \erfrontier$ contains at most $kn + 1$ elements. Because each $r \in \erfrontier$ has to be proven apart from each $q \in \erbasis$, $|\{(q,q') \in \erbasis \times \erfrontier |\ q \apart q'\}| \leq (n-1)(kn + 1)$. When simplifying this all, $N_\er(\Obs) \leq \bigO(kn^2)$. 

    \myparagraph{(3) Combining everything.}
    The preconditions on the rules never block the algorithm and when the norm $N_\er(T)$ cannot be increased further, only the equivalence rule can be applied which is guaranteed to terminate the algorithm. Therefore, the correct Mealy machine is learned within $kn^2$ (non-equivalence) rule applications which each require at most one OQ, leading to $\bigO(kn^2)$ OQs. Every (non-terminating) application of the equivalence rule leads to a new basis state. As there are at most $n$ basis states and we start with $\erbasis = \{q_0\}$, there can be at most $n-1$ applications of the equivalence rule. Each call requires at most $\log m$ OQs (\Cref{lem:proc_counter_er}). Thus, $\erlsharp$ requires $\bigO(kn^2 + n \log m)$ OQs and at most $n - 1$ EQs.
\end{proof}

\subsubsection{Proof \Cref{lem:mostly_learning}}
\begin{proof}
  We prove that $\eralsharp$ learns $\Sys$ within $\bigO(kn^2 + no^2 + n \log m)$ OQs and at most $n - 1$ EQs by:
  \begin{enumerate}
      \item Proving every rule application in $\eralsharp$ increases the following norm $N_{\er,A}(\Obs)$.
      \begin{equation*}
      \scalebox{0.77}{
          $N_{\er,A}(\Obs) = \underbrace{|\erbasis|(|\erbasis| + 1)}_{N_Q(\Obs)} + |\underbrace{\{(q,i)\in \basis \times I\mid \delta^{\Obs}(q,i){\downarrow}\}}_{N_{\downarrow(\Obs)}}| + |\underbrace{\{(q,q') \in \erbasis \times \erfrontier |\ q \apart q'\}}_{N_{\apart}(\Obs)}|$ 
          }
      \end{equation*}
      \begin{equation*}
          \scalebox{0.77}{
          $\qquad+ |\underbrace{\{(q,r) \in \erbasis \times \erfrontier \mid \alpha \text{ with } \sigma = \mathsf{sep}(W,\delta^\R(\access^\Obs(q)),\delta^\R(\access^\Obs(r))), \rho, \rho' \text{ prefixes of } \sigma\}}_{N_{(\erbasis \times \erfrontier){\downarrow}}}|$        
          }
      \end{equation*}
      \begin{equation*}
          \scalebox{0.77}{
          $\qquad\qquad+ |\underbrace{\{(q,p) \in \basis \times Q^\R \times Q^\R \mid \delta^\Obs(q,\sigma){\downarrow} \lor \lambda^\Obs(q,\rho)=\er, \sigma = \mathsf{sep}(W,p,p'), \rho \text{ prefix of } \sigma \}}_{N_{(\basis \times Q^\R \times Q^\R){\downarrow}}(\Obs)}|$        
          }
      \end{equation*}
      \begin{equation*}
          \scalebox{0.77}{
          $\qquad\qquad+ |\underbrace{\{(q,p) \in (\erbasis \cup \erfrontier) \times Q^\R  \mid \exists \sigma \in I^*, \mathsf{last}(\lambda^\Obs(q,\sigma))=\er \leftrightarrow \delta^\R(p,\sigma) \in Q^\R_F \}}_{N_{(\erbasis \times Q^{\R}){\leftrightarrow}}(\Obs)}|$
          }
      \end{equation*}
      With 
      $\alpha = (\delta^\Obs(q,\sigma){\downarrow} \lor \lambda^\Obs(q,\rho)=\er) \land (\delta^\Obs(r,\sigma){\downarrow} \lor \lambda^\Obs(r,\rho')=\er)$.
      \item Next, we prove that this norm is bounded by $|\Sys|$, $|I|$ and $|\K|$ and derive the maximum number of rule applications.
      \item We combine all previous parts to derive the complexity.
  \end{enumerate}

  \myparagraph{(1) Every rule application in $\eralsharp$ increases norm $N_{\er,A}(\Obs)$.} Let $\erbasis,\erfrontier,\Obs$ indicate the values before and $\erbasis',\erfrontier',\Obs'$ denote the values after the a rule application.
  This proof is similar to the termination proof of~\cite{DBLP:conf/fm/KrugerJR24} of $\alsharp$.
  \begin{description}
  \item[Rebuilding] Let $q, q' \in \erbasis$, $i \in I$ with $(\delta^\Obs(q,i){\uparrow} \lor \delta^\Obs(q,i) \in \erfrontier)$, $\neg(q' \apart \delta^{\Obs}(q,i))$ and $\access^\Obs(q), \access^\Obs(q)i \in P$. 
  Moreover, let $\sigma$ be the unique separating sequence $\mathsf{sep}(W,\delta^\R(\access^\Obs(q)),\delta^\R(\access^\Obs(r)))$. Note that $\sigma$ exists because $\access^\Obs(q)$ and $\access^\Obs(q)i$ are both in the minimal state cover of $\R$.
  The output queries $\textsc{OQ}_\er(\access^{\Obs}(q)i\sigma)$ and $\textsc{OQ}_\er(\access^{\Obs}(q')\sigma)$ are executed. We perform a case distinction based on $(\delta^\Obs(q,i){\uparrow} \lor \delta^\Obs(q,i) \in \erfrontier)$.
  
  \begin{itemize}
    \item $\delta^\Obs(q,i){\uparrow}$, it follows trivially that
      \[ N_{\converges}(\Obs') \supseteq N_{\converges}(\Obs) \cup \{(q,i)\}. \]
      All other norms stay the same or increase. Thus, $N_{\er,A}(\Obs') \ge N_{\er,A}(\Obs) + 1$.
    \item $\delta^\Obs(q,i) \in \erfrontier$. 
    Because $\sigma$ guarantees progress, there exists $\rho,\rho'$ maximal prefixes of $i\sigma,\sigma$ such that $\delta^{\Obs}(q,\rho){\downarrow}$ and $\delta^{\Obs}(q',\rho'){\downarrow}$ with
    $(\delta^\Obs(q,i\sigma){\uparrow} \land \delta^{\Obs}(q,\rho)\neq\er) \lor (\delta^\Obs(q',\sigma){\uparrow} \land \delta^{\Obs}(q',\rho')\neq\er)$. Thus, one of the performed queries must extend $\Obs$.

    Suppose $\mu$ and $\mu'$ the longest prefix of $i\sigma$ and $\sigma$ (respectively) such that $\delta^{\Obs'}(q,\mu){\downarrow}$ and $\delta^{\Obs'}(q',\mu'){\downarrow}$, then $(\delta^{\Obs'}(q',\sigma){\downarrow} \lor \lambda^{\Obs'}(q',\mu)=\er) \land (\delta^{\Obs'}(q,i\sigma){\downarrow} \lor \lambda^{\Obs'}(q,\mu')=\er)$. From this it follows that,

        \[ N_{(\erbasis\times \erfrontier)\converges}(\Obs') \supseteq N_{(\erbasis\times\erfrontier)\converges}(\Obs) \cup \{(q',\delta^{\Obs}(q,i))\}. \]

     Afterwards, the rebuilding rule cannot be applied again with $q'$ and $\delta^{\Obs}(q,i)$.
     In some cases, we might find that $\delta^{\Obs}(q,i) \apart q'$ which indicates 
      \[ N_{\apart}(\Obs') \supseteq N_{\apart}(\Obs) \cup \{(q',\delta^{\Obs}(q,i))\} \] Otherwise $N_{\apart}(\Obs') \supseteq N_{\apart}(\Obs)$.
      All other norms stay the same or increase. Thus, $N_{\er,A}(\Obs') \ge N_{\er,A}(\Obs) + 1$.
  \end{itemize}
  In both cases, $N_{\er,A}(\Obs') \ge N_{\er,A}(\Obs) + 1$.
  \item[Prioritized promotion] 
  Let $r \in \erfrontier$. Suppose $r$ is isolated and $\accessT(r) \in P$. State $r$ is moved from $\frontier$ to $basis$, which implies $\erbasis' := \erbasis \cup \{r\}$, thus
  \begin{align*}
    N_Q(\Obs') &= |\erbasis'| \cdot (|\erbasis'|+1)
    = (|\erbasis|+1) \cdot (|\erbasis|+1+1)\\
                &= (|\erbasis|+1)\cdot |\erbasis| + 2(|\erbasis|+1)
                  = N_Q(\Obs) + 2|\erbasis| + 2
  \end{align*}
  
  Because we move elements from the frontier to the basis, we find
  \[ N_{\apart}(\Obs') ~~\supseteq~~ N_{\apart}(\Obs) \setminus (\erbasis\times \{r\}) \] 
  \[ N_{(\erbasis\times\erfrontier)\converges}(\Obs) ~~\supseteq~~ N_{(\erbasis\times\erfrontier)\converges}(\Obs) \setminus (\erbasis\times \{r\})
  \]
  and thus
  \[ |N_{\apart}(\Obs')| ~~\ge~~ |N_{\apart}(\Obs)| - |\erbasis| \] 
  \[ |N_{(\erbasis\times \erfrontier)\converges}(\Obs')| ~~\ge~~ |N_{(\erbasis\times\erfrontier)\converges}(\Obs)| - |\erbasis|
  \]
  All other norms stay the same or increase. The total norm increases because
    \[ N_{\er,A}(\Obs') \ge N_{\er,A}(\Obs)+ 2\mid \erbasis\mid +~2~- \mid \erbasis\mid - \mid \erbasis\mid ~\ge N_{\er,A}(\Obs)+ 2 \]  
  \item[Promotion] Analogous to the prioritized promotion rule.
  \item[Extension] Analogous to the extension rule from \Cref{lem:error_learning}. The new components of the norm stay unchanged or rise, thus the total norm rises.
  \item[Separation] Analogous to the separation rule from \Cref{lem:error_learning}. The new components of the norm stay unchanged or rise, thus the total norm rises.
  \item[Match separation]  Let $q \in \erbasis$, $p \in Q^{\R}, i \in I$, $\sigma \in I^*$, $\delta^{\Obs}(q,i)=r \in \erfrontier$ and $\delta^{\R}(p,i)=p'$. Suppose 
  \begin{enumerate}
    \item $q \matches p$, 
    \item There is no $\rho \in I^*$ such that $\mathsf{last}(\lambda^\Obs(q,\rho))=\er \leftrightarrow \delta^\R(p,\rho) \in Q^\R_F$,
    \item There is no $s \in \erbasis$ such that $p' \matches s$,
    \item There exists $q' \in \erbasis$ such that 
    $\neg(r \apart q')$
    \item There exists $\sigma$ such that $\last(\lambda^\Obs(q',\sigma))=\er \leftrightarrow \delta^\R(p',\sigma)\in Q^\R_F$.
  \end{enumerate}
  Choose $\sigma$ such that $\last(\lambda^\Obs(q',\sigma))=\er \leftrightarrow \delta^\R(p',\sigma)\in Q^\R_F$.
  We perform a case distinction on $\last(\lambda^\Obs(q',\sigma))=\er$. 
  \begin{itemize}
    \item $\last(\lambda^\Obs(q',\sigma))=\er$ and thus $\delta^\R(p',\sigma)\in Q^\R_F$. After \textsc{OQ}$_\er(\accessT(q)i\sigma)$ we find either
    \begin{itemize}
      \item $r \apart q'$ if after the output query $\delta^\Obs(q,i\sigma){\downarrow}$ and $\last(\lambda^\Obs(q,i\sigma)) \neq \er$. Thus, $N_{\apart}(\Obs') \supseteq N_{\apart}(\Obs) \cup \{(q',r)\}$.
      \item Otherwise, there exists a prefix $\rho$ of $\sigma$ such that $\last(\lambda^\Obs(q,i\rho))=\er \leftrightarrow \delta^\R(p,i\rho) \in Q^\R_F$. Thus,
      $N_{(\erbasis \times Q^{\R}){\leftrightarrow}}(\Obs') \supseteq N_{(\erbasis \times Q^{\R}){\leftrightarrow}}(\Obs) \cup \{(q,p)\}$.
    \end{itemize}
    \item $\last(\lambda^\Obs(q',\sigma))\neq\er$ and thus $\delta^\R(p',\sigma)\notin Q^\R_F$. After \textsc{OQ}$_\er(\accessT(q)i\sigma)$ we find either
    \begin{itemize}
      \item $\delta^\Obs(q,i\sigma){\uparrow}$ or $\delta^\Obs(q,i\sigma){\downarrow}$ with $\lambda^\Obs(q,i\sigma)=\er$. From both it follows that
      $N_{\apart}(\Obs') \supseteq N_{\apart}(\Obs) \cup \{(q',r)\}$.
      \item $\delta^\Obs(q,i\sigma){\downarrow}$ with $\lambda^\Obs(q,i\sigma)\neq\er$ which implies $\mathsf{last}(\lambda^\Obs(q,i\sigma))\neq\er \leftrightarrow \delta^\R(p,i\sigma) \notin Q^\R_F$.
    \end{itemize}
  \end{itemize}
  The other components of the norm stay unchanged or rise, thus the total norm rises.
  \item[Match refinement] Let $q \in B$, $p, p' \in Q^\R$ and $q \matches p$ and $q \matches p'$. After executing $\textsc{OQ}_\er(\access^\Obs(q)\sigma)$ with $\sigma$ the unique witness between $p$ and $p'$, we find that 
  \[ N_{(\basis \times Q^\R \times Q^\R){\downarrow}}(\Obs') \supseteq N_{(\basis \times Q^\R \times Q^\R){\downarrow}}(\Obs') \cup \{(q,p,p')\} \]
  In some cases, we might additionally find $q \notmatches p$ or $q \notmatches p'$ which does not directly increase a norm but might enable application of match seperation.
  The other components of the norm stay unchanged or increase, thus the norm rises.
  
  \item[Equivalence] Analogous to the equivalence rule from \Cref{lem:error_learning}. The new components of the norm stay unchanged or rise, thus the total norm rises.
  \end{description}

  \myparagraph{(2) Norm $N_{\er,A}(\Obs)$ is bounded by $|\Sys|$, $|I|$ and $|\K|$.}
  It must hold that $|\erbasis| \leq |\Sys|$. The rules that are also in $\erlsharp$ are bounded by $kn^2$ (see \Cref{lem:error_learning}). We introduced three new norm components: 
  \[ |N_{(\erbasis \times \erfrontier){\downarrow}}| \leq (n-1)(kn+1) \] 
  \[ |N_{(\erbasis \times Q^{\R}){\leftrightarrow}}| \leq no \]
  \[ |N_{(\basis \times Q^\R \times Q^\R){\downarrow}}| \leq no^2 \]
  When simplifying this all, $N_{\er,A}(\Obs) \leq \bigO(kn^2 + no^2)$.

  \myparagraph{(3) Combining everything.}
  The preconditions on the rules never block the algorithm and when the norm $N_{\er,A}(\Obs)$
  cannot be increased further, only the equivalence rule can be applied which is guaranteed to terminate the algorithm. Therefore, the correct Mealy machine is learned within $kn^2 + no^2$ (non-equivalence) rule applications which each require at most two OQs, leading to $\bigO(kn^2 + no^2)$ OQs. As there are at most $n$ basis states and we start with $\erbasis = \{q_0\}$, there can be at most $n-1$ applications of the equivalence rule. Each call requires at most $\log m$ OQs (see \Cref{lem:proc_counter_er}). Thus, $\eralsharp$ requires $\bigO(kn^2 + no^2 + n \log m)$ OQs and at most $n - 1$ EQs.
\end{proof}

\subsection{Proofs for Section 4.3 and 4.4}
Before we prove \Cref{thm:sound_learning}, we prove some lemmas related to counterexample processing. These are the sound variants of previous \Cref{lem:proc_counter_er} and \Cref{lem:cex_after_oq_er}.

\begin{lemma} \label{lem:proc_counter_sound}
  Let $\Obs$ be an observation tree, $\Hyp$ a Mealy machine constructed using $\textsc{BuildHypothesis}_\er(\Obs)$, $\K$ a $\er$-persistent reference for $\Sys$, and $\sigma \in I^*$ such that $\delta^\Hyp(\sigma) \apart_{\sound,\K} \delta^\Obs(\sigma)$. Then \textsc{ProcCounterEx}$_\sound(\Hyp,\sigma)$ using $\textsc{OQ}_\sound$ terminates and is correct, i.e., after termination $\Hyp$ is not a hypothesis for $\Obs$ anymore.
\end{lemma}
\begin{proof}
This proof is similar to \Cref{lem:proc_counter_er}. We prove that using $\textsc{OQ}_\sound$ and $\sound$-apartness also results in $\Hyp$ not being a hypothesis for $\Obs$ anymore.
Let $q := \delta^\Hyp(\sigma), r:= \delta^\Obs(\sigma)$ and $\eta \vdash q \apart_{\sound,\K} r$ for some $\eta \in I^*$.

\begin{itemize}
  \item Suppose $r \in \{ \delta^\Obs(r',i) \mid r' \in \erbasis, i \in I \}$ but $r \notin \erbasis~\cup~\erfrontier$. It follows that $\lambda^\Hyp(\delta^{\Hyp}(\access(r')),i)=\er$ and $\delta^\Hyp(\delta^{\Hyp}(\access(r')),i)=q_\er=q$ (by construction). As the hypothesis and system are $\er$-persistent, there can exist no $\eta \in I^*$ such that $\eta \vdash q \apart_{\sound,\K} r$. 
  
  Because we assume $r$ exists in $\Obs$, the case that $\delta^\Obs(r',i){\uparrow}$ does not need to be considered here.
  \item Suppose $r \in \erbasis \cup \erfrontier$. 
  We decompose $\sigma = \alpha i$ into $\alpha \in I^*$, $i \in I$. Let $q' := \delta^\Hyp(\alpha), r := \delta^\Obs(\alpha)$. Since $\Obs$ is a tree, $\alpha = \access(r')$ and thus $q' = r'$. We have $q' \rightarrow^{i} q$ (in $\Hyp$) and $r' \rightarrow^{i} r$ (in $\Obs$).
  The basis must be $\sound$-adequate before applying the equivalence rule, thus it follows that $\neg(q \apart_{\sound,\K} r)$.
  However, we now have $q \apart_{\sound,\K} r$. Thus, $\Hyp$ is not a hypothesis for $\Obs$.
  
  \item Suppose $r \notin \erbasis \cup \erfrontier$. We decompose  $\sigma$ into $\sigma_1\sigma_2$ such that $q' := \delta^\Hyp(\sigma_1) \in \erbasis$ and $r' := \delta^\Obs(\sigma_1)$ as in \Cref{alg:counterexamplebs}. We perform $\textsc{OQ}_\sound(\access(q')\sigma_2\eta)$.
  
  Suppose afterwards $\delta^\Obs(q',\sigma_2\eta){\downarrow}$. We perform a case distinction based on whether $q' \apart_{\sound,\K} r'$ holds.
  \begin{itemize}
    \item If $q' \apart_{\sound,\K} r'$, then $\delta^\Hyp(q_0,\sigma_1) \apart_{\sound,\K} \delta^\Obs(q_0,\sigma_1)$. This makes $\sigma_1$ a valid parameter for \textsc{ProcCounterEx}$_\sound$ and, by induction, $\Hyp$ is not a valid hypothesis for $\Obs$ anymore after the recursive call.
    \item If $\neg(q' \apart_{\sound,\K} r')$, then either $\lambda^\Obs(r,\eta){\uparrow}$ or $\lambda^\Obs(r,\eta)=\er$. In the other casem $\lambda^\Obs(r,\eta)\neq\er$, assumption $\neg(q' \apart_{\sound,\K} r')$ would be contradicted. In both cases, it follows that $\lambda^\Obs(q',\sigma_2\eta)=\er$ and $\lambda^\Obs(q,\eta)\neq \er$. Thus, $\eta \vdash \delta^\Hyp(\access(q')\sigma_2) \apart_{\sound,\K} \delta^\Obs(\access(q')\sigma_2)$ holds and invoking \textsc{ProcCounterEx}$_\sound$ with $\access(q')\sigma_2$ ensures $\Hyp$ is not a valid hypothesis for $\Obs$ anymore after the recursive call.
  \end{itemize}
  
  Suppose $\delta^\Obs(q',\sigma_2\eta){\uparrow}$. We perform a case distinction based on $\lambda^\Obs(r,\eta)$.
  \begin{itemize}
    \item If $\lambda^\Obs(r,\eta)=\er$, then $\access(r)\eta = \sigma_1\sigma_2\eta \in \K$. 
    From $\eta \vdash q \apart_{\sound,\K} r$ and the definition of $\apart_{\sound,\K}$, it follows that $\eta \vdash q \apart r$. Because $\eta \vdash q \apart r$ and $\lambda^\Obs(r,\eta)=\er$, it follows that $\lambda^\Obs(q,\eta)\neq\er$. 
    In that case $\eta \vdash \delta^\Hyp(\access(q')\sigma_2) \apart_{\sound,\K} \delta^\Obs(\access(q')\sigma_2)$ and invoking \textsc{ProcCounterEx}$_\sound$ with $\access(q')\sigma_2$ ensures $\Hyp$ is not a valid hypothesis for $\Obs$ anymore after the recursive call.
    \item If $\lambda^\Obs(r,\eta)\neq\er$, then we have $\sigma_2\eta \vdash \delta^\Hyp(\sigma_1) \apart_{\sound,\K} \delta^\Obs(\sigma_1)$. This makes $\sigma_1$ a valid parameter for \textsc{ProcCounterEx}$_\sound$ and, by induction, $\Hyp$ is not a valid hypothesis for $\Obs$ anymore after the recursive call.
    \item If $\lambda^\Obs(r,\eta){\uparrow}$, then $\access(r)\eta = \sigma_1\sigma_2\eta \notin \K$. By $\eta \vdash q \apart_{\sound,\K} r$, it must hold that $\lambda^\Obs(q,\eta)\neq \er$. In that case $\eta \vdash \delta^\Hyp(\access(q')\sigma_2) \apart_{\sound,\K} \delta^\Obs(\access(q')\sigma_2)$ and invoking \textsc{ProcCounterEx}$_\sound$ with $\access(q')\sigma_2$ ensures $\Hyp$ is not a valid hypothesis for $\Obs$ anymore after the recursive call.
  \end{itemize}
\end{itemize}
\end{proof}

\begin{lemma} \label{lem:cex_after_oq_sound}
  Let $\Sys$ be $\er$-persistent Mealy machine, $\K$ a reference, $\Obs$ an observation tree for $\Sys$, $\Hyp$ constructed using $\textsc{BuildHypothesis}_\er(\Obs)$, $\rho \in \K$ such that $\lambda^\Hyp(\rho) \neq \lambda^\Sys(\rho)$. After executing $\textsc{OQ}_\sound(\rho,\K)$, there exists some $\mu \in I^*$ such that $\lambda^\Hyp(\mu) \apart_{\sound,\K} \lambda^\Obs(\mu)$.
\end{lemma}
\begin{proof}
  Because $\rho \in \K$, this proof reduces to the proof of \Cref{lem:cex_after_oq_er}.
\qed
\end{proof}

\subsubsection{Proof of \Cref{thm:sound_learning}}
\begin{proof}
    We prove that $\soundlsharp$ learns a Mealy machine $\Hyp$ such that $\Hyp \sim_{\K} \Sys$ and $\lambda^{\Sys}(\sigma) = q_{\er}$ for all $\sigma \notin \K$. This $\Hyp$ is learned within $\bigO(kn^2 + n \log m)$ OQs and at most $n - 1$ EQs by:
    \begin{enumerate}
        \item Proving every rule application in $\soundlsharp$ increases the following norm $N_{\sound}(\Obs)$.
        \begin{equation*}
            \scalebox{0.77}{
                $N_{\sound}(\Obs) = \underbrace{|\erbasis|(|\erbasis| + 1)}_{N_Q(\Obs)} + |\underbrace{\{(q,i)\in \basis \times I\mid \delta^{\Obs}(q,i){\downarrow}\}}_{N_{\downarrow(\Obs)}}| + |\underbrace{\{(q,q') \in \erbasis \times \erfrontier |\ q \apart_{\sound,\K} q'\}}_{N_{\apart_{\sound,\K}}(\Obs)}|$ 
                }
            \end{equation*}
            \begin{equation*}
                \scalebox{0.77}{
                $\qquad+ |\underbrace{\{(q,r) \in \erbasis \times \erfrontier \mid \alpha \text{ with } \sigma = \mathsf{sep}(W,\delta^\R(\access^\Obs(q)),\delta^\R(\access^\Obs(r))), \rho, \rho' \text{ prefixes of } \sigma\}}_{N_{(\erbasis \times \erfrontier){\downarrow}}}|$        
                }
            \end{equation*}
            With 
            $\alpha = (\access^\Obs(q)\sigma \in \K \land (\delta^\Obs(q,\sigma){\downarrow} \lor \lambda^\Obs(q,\rho)=\er)) $ $\lor$\\ $(\access^\Obs(r)\sigma \in \K \land (\delta^\Obs(r,\sigma){\downarrow} \lor \lambda^\Obs(r,\rho')=\er))$.
          \item Next, we prove that this norm is bounded by $|\Sys|$ and $|I|$ and derive the maximum number of rule applications.
          \item We combine all previous parts to derive the complexity.
          \item We prove $\Hyp \sim_{\K} \Sys$ and $\lambda^{\Sys}(\sigma) = q_{\er}$ for all $\sigma \notin \K$
    \end{enumerate}

    \myparagraph{(1) Every rule application in $\soundlsharp$ increases norm $N_\sound(\Obs)$.} Let $\erbasis,\erfrontier,\Obs$ indicate the values before and $\erbasis',\erfrontier',\Obs'$ denote the values after the a rule application. 
    \begin{description}
        \item[Rebuilding] Let $q, q' \in \erbasis$, $i \in I$ with $(\delta^\Obs(q,i){\uparrow} \lor \delta^\Obs(q,i) \in \erfrontier)$, $\neg(q' \apart_{\sound,\K} \delta^{\Obs}(q,i))$ and $\access^\Obs(q), \access^\Obs(q)i \in P$. Moreover, 
        let $\sigma$ be the unique separating sequence $\mathsf{sep}(W,\delta^\R(\access^\Obs(q)),\delta^\R(\access^\Obs(r)))$. Suppose $\rho$ and $\rho'$ the longest prefix of $i\sigma$ and $\sigma$ (respectively) such that $\delta^\Obs(q,\rho){\downarrow}$ and $\delta^\Obs(q',\rho){\downarrow}$, then 
        ($\access^\Obs(q)i\sigma \in \K \land \delta^\Obs(q,i\sigma){\uparrow} \land \delta^{\Obs}(q,\rho)\neq\er)$ or $(\access^T(q')\sigma \in \K \land \delta^\Obs(q',\sigma){\uparrow} \land \delta^{\Obs}(q',\rho')\neq\er)$.
        
        The output queries $\textsc{OQ}_{\sound}(\access^{\Obs}(q)i\sigma)$ and $\textsc{OQ}_{\sound}(\access^{\Obs}(q')\sigma)$ are executed and are guaranteed to extend $\Obs$ using a similar reasoning as in \Cref{lem:mostly_learning}. Thus, 
        \[ N_{(\erbasis\times \erfrontier)\converges}(\Obs') \supseteq N_{(\erbasis\times\erfrontier)\converges}(\Obs) \cup \{(q',\delta^{\Obs}(q,i))\}. \]
        Additionally, an increasement $N_{\converges}(\Obs') \supseteq N_{\converges}(\Obs) \cup \{(q,i)\}$ is guaranteed if $\delta^\Obs(q,i){\uparrow}$ and $\access^\Obs(q)i \in \K$. 
        In some cases, we may find $N_{\apart_{\sound,\K}}$ increases because $q' \apart_{\sound,\K} \delta^\Obs(q,i)$.
        
        In all cases, $N_\sound(\Obs') \ge N_\sound(\Obs) + 1$.
        \item[Prioritized promotion] 
        Let $r \in \erfrontier$. Suppose $r$ is isolated and $\accessT(r) \in P$. State $r$ is moved from $\frontier$ to $basis$, which also implies $\erbasis' := \erbasis \cup \{r\}$, thus
        \begin{align*}
          N_Q(\Obs') &= |\erbasis'| \cdot (|\erbasis'|+1)
          = (|\erbasis|+1) \cdot (|\erbasis|+1+1)\\
                     &= (|\erbasis|+1)\cdot |\erbasis| + 2(|\erbasis|+1)
                       = N_Q(\Obs) + 2|\erbasis| + 2
        \end{align*}
        
        Because we move elements from the frontier to the basis, we find
        \[ N_{\apart_{\sound,\K}}(\Obs') ~~\supseteq~~ N_{\apart_{\sound,\K}}(\Obs) \setminus (\erbasis\times \{r\}) \] 
        \[ N_{(\erbasis\times\erfrontier)\converges}(\Obs) ~~\supseteq~~ N_{(\erbasis\times\erfrontier)\converges}(\Obs) \setminus (\erbasis\times \{r\})
        \]
        and thus
        \[ |N_{\apart_{\sound,\K}}(\Obs')| ~~\ge~~ |N_{\apart_{\sound,\K}}(\Obs)| - |\erbasis| \] 
        \[ |N_{(\erbasis\times \erfrontier)\converges}(\Obs')| ~~\ge~~ |N_{(\erbasis\times\erfrontier)\converges}(\Obs)| - |\erbasis|
        \]
        All other norms stay the same or increase. The total norm increases because
         \[ N_\sound(\Obs') \ge N_\sound(\Obs)+ 2\mid \erbasis\mid +~2~- \mid \erbasis\mid - \mid \erbasis\mid ~\ge N_\sound(\Obs)+ 2 \]  
        \item[Promotion] Analogous to the prioritized promotion rule.
        \item[Extension] If $\delta^{\Obs}(q,i){\uparrow}$ for some $q\in \basis$ and $i\in I$ and $\access^\Obs(q)i \in \K$, we execute $\textsc{OQ}_\sound(\mathsf{access}(q) \, i)$. This may result in either $\lambda^{\Obs}(q,i) = \er$ or $\lambda^{\Obs}(q,i) \neq \er$ as $\access^\Obs(q)i \in \K$ but in both cases,
          \[
            N_\downarrow(\Obs') = N_{\downarrow}(\Obs)\cup\{(q,i)\}.
          \]
          All other norms stay the same or increase. Thus $N_\sound(\Obs') \ge N_\sound(\Obs) + 1$.
        \item[Separation] Let $q\in \erfrontier$ and distinct $r,r'\in \erbasis$
          with $\neg(q\apart_{\sound,\K} r)$ and $\neg(q\apart_{\sound,\K} r')$. The algorithm performs 
          the query $\textsc{OQ}_{\sound}(\mathsf{access}(q)\sigma)$. 
          We perform a case distinction on $r \apart_{\sound,\K} r'$:
          \begin{itemize}
            \item $\sigma \vdash r \apart r'$. In particular this means $\delta^\Obs(r,\sigma){\downarrow}$ and $\delta^{\Obs}(r',\sigma){\downarrow}$. Moreover, there exists some prefix $\rho$ of $\sigma$ inequivalent to $\varepsilon$ such that $\access^\Obs(q)\rho \in \K$, otherwise $\neg(q\apart_{\sound,\K} r)$ or $\neg(q\apart_{\sound,\K} r')$. Thus, executing $\textsc{OQ}_{\sound}(\mathsf{access}(q)\sigma)$ extends the observation tree and leads to either $q \apart_{\sound,\K} r$ or $q \apart_{\sound,\K} r'$.
            \item $\last(\lambda^\Obs(r,\sigma)) \neq \er \land \access^\Obs(r')\sigma \notin \K$. We must have $\access^\Obs(q)\sigma \in \K$, otherwise $r \apart q$. Thus, $\textsc{OQ}_{\sound}(\mathsf{access}(q) \; \sigma)$ extends the observation tree and leads to either $q \apart r$ or $q \apart_{\sound,\K} r'$.
            \item $\last(\lambda^\Obs(r',\sigma)) \neq \er \land \access^\Obs(r)\sigma \notin \K$. Analogous.
          \end{itemize}
        In any case, $|N_{\apart_{\sound,\K}}(\Obs')| \ge |N_{\apart_{\sound,\K}}(\Obs)| + 1$. The other components of the norm stay unchanged or increase, thus the norm rises.
        \item[Equivalence] If the algorithm does not terminate after the equivalence rule, we show that the norm increases because a frontier state becomes isolated.
        Because we restrict the EQ to $\K$, it holds that $\rho \in \K$ such that $\lambda^\Hyp(\rho) \neq \lambda^\Sys(\rho)$. By \Cref{lem:cex_after_oq_sound}, after executing $\textsc{OQ}_\sound(\rho,\K)$, there exists some $\mu \in I^*$ such that $\lambda^\Hyp(\mu) \apart_{\sound,\K} \lambda^\Obs(\mu)$. Thus, calling $\textsc{ProcCounterEx}_\sound$ with $\sigma$ the shortest prefix of $\mu$ such that $\delta^\Hyp(\sigma) \apart_{\sound,\K} \delta^\Obs(\sigma)$ results in $\Hyp$ not longer being a hypothesis for $\Obs$ (\Cref{lem:proc_counter_sound}). $\Hyp$ can only no longer be a hypothesis for $\Obs$ if a frontier state becomes isolated. In particular, this implies 
        \[
            (q,r) \in N_{\apart_{\sound,\K}}(\Obs') \setminus N_{\apart_{\sound,\K}}(\Obs).
            \tag*{\qed}
        \]
        As all other components of the norm remain the same, $N_\sound(\Obs') \geq N_\sound(\Obs) + 1$.
        \end{description}

    \myparagraph{(2)+(3) Norm $N_\sound(\Obs)$ is bounded by $|\Sys|$ and $|I|$.}
    This part of the proof is evident from \Cref{lem:mostly_learning}, the complexity is $\bigO(kn^2 + n \log m)$ OQs and $n - 1$ EQs.

    \myparagraph{(4) $\Hyp \sim_{\K} \Sys$ and $\lambda^{\Sys}(\sigma) = q_{\er}$ for all $\sigma \notin \K$}
    From $\textsc{EQ}$ on $\K$, we know that after termination $\Hyp \sim_\K \Sys$ must hold. 
    Procedure $\CheckSoundness$ guarantees that $\K$ is sound for any $\Hyp$ used to pose an $\textsc{EQ}$ in the equivalence rule.
    From this, it follows that for any $\sigma \notin \K$, $\lambda^\Hyp(\sigma) = \er$.
    By construction of \textsc{BuildHypothesis}$_\er$, for any $\sigma \in I^*$ such that $\lambda^\Hyp(\sigma) = \er$, $\delta^\Hyp(\sigma) = q_\er$.
    Thus, $\Hyp \sim_{\K} \Sys$ and $\lambda^{\Sys}(\sigma) = q_{\er}$ holds for all $\sigma \notin \K$ after termination of $\soundlsharp$.
\qed
\end{proof}

\subsubsection{Proof of \Cref{lem:rebuilding}}~\newline
  Let $\Obs$ be the observation tree after executing $T_\K$. 
  Because $P$ is an \emph{accepting} state cover, it holds that $p \in \K$ for all $p \in P$. 
  By definition of \Cref{eq:t}, there exists a sequence $\sigma \in T_\K$ with prefix $p$ for all $p \in P$.
  Combining this with $p \in \K$ for all $p \in P$ and the fact that all sequences in $T_\K$ are executed with $\textsc{OQ}_{\sound}$, it must hold that $\delta^\Obs(p){\downarrow}$ for all $p \in P$. 

  For every distinct pair of prefixes $p, q \in P$, it holds that $\delta^\R(p) \apart \delta^\R(q)$ because $\R$ is minimal. 
  By definition of a separating family, there exists some $\sigma \in W_{\delta^\R(p)} \cap W_{\delta^\R(q)}$ with $\sigma \vdash \delta^\R(p) \apart \delta^\R(q)$. 
  When combining this with \Cref{eq:t}, it follows that there exists sequences $p\sigma, q\sigma \in T_\K$ with $\sigma \vdash \delta^\R(p) \apart \delta^\R(q)$. 
  It follows that $p\sigma \in \K$ and $q\sigma \notin \K$ or vice versa. 
  Assume $p\sigma \in \K$ and $q\sigma \notin \K$, the other case is analogous.
  From assumption $p\sigma \in \K$ and completeness of $\K$, it follows that $\last(\lambda^\Obs(p\sigma))  \neq \er$.
  By definition of $\sound$-apartness (\Cref{def:s_ap}), two states $p'$ and $q'$ are $\sound$-apart if $\sigma \vdash p' \apart q'$ or $\last(\lambda^{\Obs}(p',\sigma)) \neq \er$ and $\access^{\Obs}(q')\sigma \notin \K$. 
  The second case holds for our scenario as $\last(\lambda^\Obs(p\sigma)) = \last(\lambda^\Obs(\delta^\Obs(p),\sigma)) \neq \er$ and $\access^{\Obs}(q)\sigma = q\sigma \notin \K$. Thus, for every distinct pair of prefixes $p$ and $q$, we have $\delta^\Obs(p) \apart_{\sound,\K} \delta^\Obs(q)$.

  When combining this with the fact that $\delta^\Obs(p){\downarrow}$ for all $p \in P$, we find that all states reached by $p \in P$ are pairwise $\sound$-apart and thus pairwise $\sound\complete$-apart. \qed

\begin{lemma} \label{lem:proc_counter_complete}
    Let $\Obs$ be an observation tree, $\Hyp$ a Mealy machine constructed using \textsc{BuildHypothesis}\mbox{}$_\er(\Obs)$, $\K$ a reference and $\sigma \in I^*$ such that $\delta^\Hyp(\sigma) \apart_{\sound\complete,\K} \delta^\Obs(\sigma)$. Then \textsc{ProcCounterEx}\mbox{}$_{\sound\complete}(\Hyp,\sigma)$ using \textsc{OQ}\mbox{}$_\sound$ terminates and is correct, i.e., after termination $\Hyp$ is not a hypothesis for $\Obs$ anymore.
  \end{lemma}
  \begin{proof}
  This proof is similar to \Cref{lem:proc_counter_sound}, the only difference is that it uses $\sound\complete$-apartness which leads to some cases not being possible in the case distinctions.
  \end{proof}

\subsubsection{Proof of \Cref{thm:soundcomplete_learning}}
\begin{proof}
    We prove that $\sclsharp$ learns a Mealy machine $\Hyp$ such that $\Hyp \sim_{\K} \Sys$ and $\lambda^{\Sys}(\sigma) = q_{\er}$ for all $\sigma \notin \K$. First, we note that if $\K$ is complete, the violation rule will never be triggered. Thus, $\Hyp$ is learned within $\bigO(kn^2 + n \log m)$ OQs and at most $n - 1$ EQs if $\K$ is complete for $\Sys$:
    \begin{enumerate}
        \item Proving every rule application in $\sclsharp$ increases the following norm $N_{\sound}(\Obs)$.
        \begin{equation*}
            \scalebox{0.78}{
                $N_{\sound\complete}(\Obs) = \underbrace{|\erbasis|(|\erbasis| + 1)}_{N_Q(\Obs)} + |\underbrace{\{(q,i)\in \basis \times I\mid \delta^{\Obs}(q,i){\downarrow}\}}_{N_{\downarrow(\Obs)}}| + |\underbrace{\{(q,q') \in \erbasis \times \erfrontier |\ q \apart_{\sound\complete,\K} q'\}}_{N_{\apart_{\sound\complete,\K}}(\Obs)}|$ 
            }
            \end{equation*}
          \item Next, we prove that this norm is bounded by $|\Sys|$ and $|I|$ and derive the maximum number of rule applications.
          \item We combine all previous parts to derive the complexity using \Cref{lem:rebuilding}.
          \item We prove $\Hyp \sim_{\K} \Sys$ and $\lambda^{\Sys}(\sigma) = q_{\er}$ for all $\sigma \notin \K$.
    \end{enumerate}

    \myparagraph{(1) Every rule application in $\sclsharp$ increases norm $N_{\sound\complete}(\Obs)$.} Let $\erbasis,\erfrontier,\Obs$ indicate the values before and $\erbasis',\erfrontier',\Obs'$ denote the values after the a rule application. 
    \begin{description}
        \item[Promotion] 
        Let $r \in \erfrontier$ with $r$ is isolated. State $r$ is moved from $\frontier$ to $basis$, which also implies $\erbasis' := \erbasis \cup \{r\}$, thus
        \begin{align*}
          N_Q(\Obs') &= |\erbasis'| \cdot (|\erbasis'|+1)
          = (|\erbasis|+1) \cdot (|\erbasis|+1+1)\\
                     &= (|\erbasis|+1)\cdot |\erbasis| + 2(|\erbasis|+1)
                       = N_Q(\Obs) + 2|\erbasis| + 2
        \end{align*}
        
        Because we move elements from the frontier to the basis, we find
        \[ N_{\apart_{\sound\complete,\K}}(\Obs') ~~\supseteq~~ N_{\apart_{\sound\complete,\K}}(\Obs) \setminus (\erbasis\times \{r\}) \] 
        and thus
        \[ |N_{\apart_{\sound\complete,\K}}(\Obs')| ~~\ge~~ |N_{\apart_{\sound\complete,\K}}(\Obs)| - |\erbasis| \] 
        All other norms stay the same or increase. The total norm increases because
         \[ N_{\sound\complete}(\Obs') \ge N_{\sound\complete}(\Obs)+ 2\mid \erbasis\mid +~2~- \mid \erbasis\mid  ~\ge N_{\sound\complete}(\Obs)~+ \mid \erbasis\mid + ~2 \]  
        \item[Extension] If $\delta^{\Obs}(q,i){\uparrow}$ for some $q\in \basis$ and $i\in I$ and $\access^\Obs(q)i \in \K$, we execute $\textsc{OQ}_\sound(\mathsf{access}(q) \, i)$. This can only result in $\lambda^{\Obs}(q,i) \neq \er$ as $\access^\Obs(q)i \in \K$ and $\K$ is complete w.r.t. $\Sys$,
          \[
            N_\downarrow(\Obs') = N_{\downarrow}(\Obs)\cup\{(q,i)\}.
          \]
          All other norms stay the same or increase. Thus $N_{\sound\complete}(\Obs') \ge N_{\sound\complete}(\Obs) + 1$.
        \item[Separation] Let $q\in \erfrontier$ and distinct $r,r'\in \erbasis$
          with $\neg(q\apart_{\sound\complete,\K} r)$ and $\neg(q\apart_{\sound\complete,\K} r')$. The algorithm performs 
          the query $\textsc{OQ}_{\sound}(\mathsf{access}(q)\sigma)$. 
          We perform a case distinction on $r \apart_{\sound\complete,\K} r'$:
          \begin{itemize}
            \item $\sigma \vdash r \apart r'$. In particular this means $\delta^\Obs(r,\sigma){\downarrow}$ and $\delta^{\Obs}(r',\sigma){\downarrow}$. Moreover, there exists some prefix $\rho$ of $\sigma$ inequivalent to $\varepsilon$ such that $\access^\Obs(q)\rho \in \K$, otherwise $\neg(q\apart_{\sound\complete,\K} r)$ or $\neg(q\apart_{\sound\complete,\K} r')$. Thus, executing $\textsc{OQ}_{\sound}(\mathsf{access}(q) \sigma)$ extends the observation tree and leads to either $q \apart r$ or $q \apart r'$.
            \item $\last(\lambda^\Obs(r,\sigma)) \neq \er \land \access^\Obs(r')\sigma \notin \K$. We must have $\access^\Obs(q)\sigma \in \K$, otherwise $r \apart q$. Thus, $\textsc{OQ}_{\sound}(\mathsf{access}(q) \; \sigma)$ extends the observation tree and leads to either $q \apart r$ or $q \apart_{\sound,\K} r'$.
            \item $\last(\lambda^\Obs(r',\sigma)) \neq \er \land \access^\Obs(r)\sigma \notin \K$. Analogous.
            \item We do not consider $\last(\lambda^\Obs(r,\sigma)) = \er$ or $\last(\lambda^\Obs(r',\sigma)) = \er$ as these violate the completeness assumption.
          \end{itemize}
        In any case, $|N_{\apart_{\sound\complete,\K}}(\Obs')| \ge |N_{\apart_{\sound\complete,\K}}(\Obs)| + 1$. The other components of the norm stay unchanged or increase, thus the norm rises.
        \item[Equivalence] If the algorithm does not terminate after the equivalence rule, we show that the norm increases because a frontier state becomes isolated.
        Because we restrict the EQ to $\K$, it holds that $\rho \in \K$ such that $\lambda^\Hyp(\rho) \neq \lambda^\Sys(\rho)$. By \Cref{lem:cex_after_oq_sound}, after executing $\textsc{OQ}_\sound(\rho,\K)$, there exists some $\mu \in I^*$ such that $\lambda^\Hyp(\mu) \apart_{\sound,\K} \lambda^\Obs(\mu)$. Thus, calling $\textsc{ProcCounterEx}_\sound\complete$ with $\sigma$ the shortest prefix of $\mu$ such that $\delta^\Hyp(\sigma) \apart_{\sound\complete,\K} \delta^\Obs(\sigma)$ results in $\Hyp$ not longer being a hypothesis for $\Obs$ (\Cref{lem:proc_counter_complete}). $\Hyp$ can only no longer be a hypothesis for $\Obs$ if a frontier state becomes isolated. In particular, this implies 
        \[
            (q,r) \in N_{\apart_{\sound\complete,\K}}(\Obs') \setminus N_{\apart_{\sound\complete,\K}}(\Obs).
            \tag*{\qed}
        \]
        As all other components of the norm remain the same, $N_{\sound\complete}(\Obs') \geq N_{\sound\complete}(\Obs) + 1$.
        \end{description}

    \myparagraph{(2)+(3) Norm $N_{\sound\complete}(\Obs)$ is bounded by $|\Sys|$ and $|I|$.}
    This part of the proof is evident from \Cref{thm:sound_learning}, the output query complexity is $\bigO(kn^2 + n \log m)$ OQs and at most $n - 1$ EQs are required.
    
    \myparagraph{(4) $\Hyp \sim_{\K} \Sys$ and $\lambda^{\Sys}(\sigma) = q_{\er}$ for all $\sigma \notin \K$}
    From $\textsc{EQ}$ on $\K$, we know that after termination $\Hyp \sim_\K \Sys$ must hold. 
    Procedure $\CheckSoundness$ guarantees that $\K$ is sound for any $\Hyp$ used to pose an $\textsc{EQ}$ in the equivalence rule.
    From this, it follows that for any $\sigma \notin \K$, $\lambda^\Hyp(\sigma) = \er$.
    By construction of \textsc{BuildHypothesis}$_\er$, for any $\sigma \in I^*$ such that $\lambda^\Hyp(\sigma) = \er$, $\delta^\Hyp(\sigma) = q_\er$.
    Thus, $\Hyp \sim_{\K} \Sys$ and $\lambda^{\Sys}(\sigma) = q_{\er}$ holds for all $\sigma \notin \K$ after termination of $\sclsharp$.
\qed
\end{proof}

\subsubsection{Proof of \Cref{lem:nmino}}
\begin{proof}
  It remains to prove that at most $n - o$ EQs are required.
    From \Cref{lem:rebuilding}, we know that after executing a $T_\K$, minimal accepting state cover $P$ forms a basis. We perform a case distinction on whether $\R$ contains a sink state that is reached by only error-producing transitions.
    \begin{itemize}
      \item If $\R$ does not contain a sink state or there exists a transition that does not produce an error output with the sink state as destination, $P$ reaches all $o$ states. Because $\Sys$ contains at most $n$ states and the basis contains $o$ states after rebuilding, at most $n - o$ equivalence queries have to be posed.
      \item If $\R$ contains a sink state that is reached by only error-producing transitions, $P$ reaches $o - 1$ states. Because $\R$ contains a sink state, $\Sys$ must contain a sink state. This sink state is automatically added while constructing the hypothesis, thus at most $n - 1 - (o - 1) = n - o$ equivalence queries have to be posed.
    \end{itemize}
  \qed
\end{proof}
Before we prove that a $\Sys$ is guaranteed to be in the $k$-$A$-complete fault domain of a Mealy machine if $k$ is sufficiently large, $\K$ is sound and complete w.r.t. $\Sys$ and $A$ is a state cover, we formally define $k$-$A$-completeness~\cite{DBLP:conf/concur/VaandragerM25}.

\begin{definition} \label{def:k_compl}
  Let $k \in \mathbb{N}$ and $A \subseteq I^*$.
  A test suite $T$ is $k$-$A$-complete for $\Hyp$ if it is complete w.r.t. $\C^{k,A} = \{ \Sys \mid \text{there exists } \sigma, \mu \in A \text{ such that } \sigma \neq \mu$ $\text{and } \delta^\Sys(q_0,\sigma) \sim \delta^{\Sys}(q_0,\mu) \text{ or for all } q \in Q^{\Sys}, \text{ there exists } \sigma \in A 
  \text{ and } \rho \in I^*$ $\text{with}$ $|\rho| \leq k \text{ such that } \delta^{\Sys}(q_0, \sigma\rho) = q \}$.\looseness=-1
\end{definition}

\begin{lemma} \label{lem:k-A-proof}
    Let $\Sys$ be a $\er$-persistent Mealy machine, $\R$ a minimal DFA repesentation of a reference $\K$ and $A$ a minimal state cover of $\K$. If $\K$ is sound and complete w.r.t. $\Sys$ then $|\Sys| - |\R| \leq k$ implies $\Sys \in \C^{k,A}_{\er}$
\end{lemma}
\begin{proof}
    Because we assume $\Sys$ is $\er$-persistent and minimal, the only property we have to prove to show that $\Sys \in \C^{k,A}_{\er}$ is: for all $q \in Q^{\Sys}$, there exists $\sigma \in A$ and $\rho \in I^*$ with $|\rho| \leq k$ s.t. $\delta^{\Sys}(q_0,\sigma\rho)=q$. We first prove that $A$ reached $|\R|$ distinct states in $\Sys$. 
    
    Because we assume $\K$ is sound and complete w.r.t. any $\Sys \in \C^{k,A}_{\er}$ and $\R$ is the minimal DFA representation of $\K$, it must hold that (1) $\delta^\R(\sigma) \in Q^\R_F$ iff $\last(\lambda^{\Sys}(q_0,\sigma)) \neq \er$ for all $\sigma \in I^*$. For any two states $q, r \in Q^{\R}$, there exists some $\sigma \in I^*$ such that $\delta^{\R}(q,\sigma) \in Q^{\R}_F$ iff $\delta^{\R}(r,\sigma) \notin Q^{\R}_F$ because $\R$ is minimal. 
    Because $A$ is a minimal state cover for $\R$, it holds that (2) for any distinct $\mu, \rho \in A$, there exists some $\sigma \in I^*$ such that $\delta^{\R}(\mu\sigma) \in Q^{\R}_F$ iff $\delta^{\R}(\rho\sigma) \notin Q^{\R}_F$.
    From (1) and (2), we derive that for any distinct $\mu, \rho \in A$, there exists some $\sigma \in I^*$ such that $\last(\lambda^{\Sys}(\mu\sigma)) \neq \er$ iff $\last(\lambda^{\Sys}(\rho\sigma)) = \er$. Thus, $A$ reaches precisely $|A| = |\R|$ states in $\Sys$.
    
    Next, we prove that there exists $\sigma \in A$ and $\rho \in I^*$ with $|\rho| \leq k$ s.t. $\delta^{\Sys}(q_0,\sigma\rho)=q$ for any $q \in Q^{\Sys}$. 
    Suppose $q \in Q^{\Sys}$ is reachable with some word $\psi \in I^*$ and let $\psi$ be minimal, i.e., there is no $\mu \in I^*$ with $|\mu| < |\psi|$ such that $\delta^{\Sys}(q_0,\mu)=q$. We decompose $\psi = \sigma\rho$ with $\sigma$ the longest prefix such that $\sigma \in A$. Because $\psi$ is minimal, the path from $\delta^{\Sys}(q_0,\sigma)$ to $\delta^{\Sys}(q_0,\sigma\rho)$ has no cycles. It follows that $|\rho| \leq k$ because $A$ reaches $|\R|$ states and $|\Sys| - |\R| \leq k$. If $\rho$ would be longer, then either there $\sigma$ is not the longest prefix in $A$ or $\psi$ is not minimal. As this construct works for any $q \in Q^{\Sys}$, we are done. 
\qed
\end{proof}

\section{Model and Reference Information} \label{app:models}
\subsubsection{Parameters.}
Randomized Wp-method generates test words on-the-fly by 
\begin{enumerate}
    \item Selecting a prefix $p$ from the state cover of $\Hyp$, 
    \item Performing $k$ steps where $k$ follows a geometric distribution, and 
    \item Appending a suffix to verify whether the expected state is reached.
\end{enumerate}
We adapt this implementation to truncate test words according to $T_\er(\Hyp)$ and $T_\sound(\Hyp,\K)$. 
We set the minimal steps to 1 and the expected length to 5. 
Mixture of Experts (\moe) also uses the randomized Wp-method, exploration parameter $\gamma = 0.2$, and starts sampling experts when $|\Hyp| \geq 5$. Besides stopping early when the hypothesis and target are equivalent, for the sound algorithms $\soundlsharp$ and $\sclsharp$ we stop when the hypothesis and target are equivalent on $\K$, the language of the reference model. Note that the total symbols sometimes exceeds $10^6$, this is can be because the budget was exceed during learning or testing. We only check whether the budget has been exceeded during learning after a hypothesis is constructed and if the budget was exceeded, we return the previous hypothesis. This can sometimes lead to total symbols exceeding $10^6$. During testing, we check after every test query whether the budget has been exceeded.

\subsubsection{Error outputs.}
For the TLS benchmark, we set $\er$ to be the set of outputs that contain substring \emph{Alert Fatal}, \emph{Decryption failed} or \emph{Connection Closed} which should all lead to closed connection state according to the specification.\footnote{\url{https://datatracker.ietf.org/doc/html/rfc5246\#}}
For the ASML benchmark, we set $\er$ to be the output \emph{error}. The monitor benchmark originally learns DFA of which some rejecting states are due to an error and other due to the risk being to high, we consider the input to produce an error if the horizon is exceeded or if the input is not enabled according to the HMM.

\subsubsection{HMM Monitors.}
To obtain the HMM monitors, we learn the model using the code of \cite{DBLP:conf/atva/MaasJ25} and a timeout of one hour. The learned monitor is not necessarily the uniquely correct monitor for the HMM because there is some slack between the safe and unsafe threshold.

\subsubsection{Obtaining sound and complete references.}
To obtain sound and complete references for TLS and ASML, we construct a DFA that accepts precisely the non-error producing words by iterating over the transitions. 
To obtain the references for the monitors, we use a mapper in the code of \cite{DBLP:conf/atva/MaasJ25} that returns \emph{False} if the input exceeds the horizon or cannot occur according to the HMM and \emph{True} otherwise.
Instead of learning the reference model, it could be extracted from the HMM.
These references are all sound and complete for a monitor that is correct for the HMM. However, they are not necessarily sound and complete for intermediate hypothesis when learning the monitor. Because a timeout of 1 hour is not enough to learn a full model, the obtained references are not always sound and complete w.r.t. the learned monitor that we consider the target model in this paper. All the airport models were learned completely, for these models we use the reference extracted from the HMM. For the other models, we extract a sound and complete model from the (known) SUL. 

\subsubsection{Obtaining Chat-GPT references.}
To obtain $P_{GPT}$ for Experiment 2, we used a fresh ChatGPT account and posed the following prompt:
\textit{
    Make a DFA (in .dot format and readable with graphviz) of the TLS 1.2 protocol client side according to the specifications that can be found in RFC 5246. Only transitions with the following inputs are allowed {"ApplicationData", "ApplicationDataEmpty", "CertificateRequest", 
    "ChangeCipherSpec", "EmptyCertificate", "Finished", 
    "HeartbeatRequest", "HeartbeatResponse", "ServerCertificate", 
    "ServerHelloDHE", "ServerHelloDone", "ServerHelloRSA"}. 
    Each state should have a transition for every input. If the input leads to an error according to the specification, the transition should go to a sink state. This sink state is the only state with shape 'circle', all other states should have shape 'double circle'.
}
Afterwards, small changes were made to ensure that the file is readable for AALpy. For 2/18 models, the chatGPT generated model was sound.

\subsubsection{Obtaining passive references.}
To obtain passive reference models for the TLS models, we use 10000 randomly generated traces of length between 20 and 30 and use the AALpy method \textsc{Run\_GSM} to create a Mealy machine. We then transform this Mealy machine into a reference that accepts precisely the non-error producing words. All non-defined transitions are considered non-error producing self-loops to ensure maximal soundness. For 5/18 models, the passively learned model was sound.

\begin{table}[htbp]
    \centering
    \caption{Information on the ASML benchmark and associated references.}
\begin{tabular}{l|rrrrrrrrrrr}
    model & m54 & m41 & m199 & m196 & m158 & m159 & m76 & m45 & m164 & m106 & m34 \\ \hline
    states & 27 & 25 & 27 & 81 & 28 & 30 & 210 & 184 & 43 & 79 & 115 \\ 
    reference states & 16 & 24 & 25 & 38 & 18 & 19 & 57 & 47 & 14 & 65 & 74 \\
\end{tabular}
\end{table}

\begin{table}[htbp]
    \centering
    \caption{Information on the monitor benchmark and associated references.}
    \begin{tabular}{lll|rrr}
        \toprule
        model & parameters & horizon & states & inputs & reference states \\
        \midrule
        airportA-3 & DMAX=3PMAX=3 & 8 & 61 & 18 & 13 \\
        airportA-3 & DMAX=3PMAX=3 & 9 & 56 & 18 & 14 \\
        airportA-7 & DMAX=3PMAX=3 & 8 & 46 & 18 & 16 \\
        airportA-7 & DMAX=3PMAX=3 & 9 & 50 & 18 & 16 \\
        airportB-3 & DMAX=3PMAX=3 & 8 & 70 & 18 & 14 \\
        airportB-3 & DMAX=3PMAX=3 & 9 & 59 & 18 & 13 \\
        airportB-7 & DMAX=3PMAX=3 & 8 & 46 & 18 & 16 \\
        airportB-7 & DMAX=3PMAX=3 & 9 & 56 & 18 & 16 \\
        evade & - & 6 & 27 & 325 & 23 \\
        evade & - & 7 & 21 & 325 & 21 \\
        hidden\_incentive & - & 8 & 75 & 100 & 65 \\
        hidden\_incentive & - & 9 & 50 & 100 & 50 \\
        refuel & N=3ENERGY=3 & - & 45 & 63 & 37 \\
        refuelB & N=3ENERGY=3 & - & 28 & 63 & 26 \\
        \bottomrule
        \end{tabular}
\end{table}

\begin{table}[htbp]
    \centering
    \caption{Information on the TLS benchmark and references.}
    \label{tab:soundcompltable}
    \resizebox{0.9\textwidth}{!}{
\begin{tabular}{lrr|rrrr}
    \toprule
    model & states & inputs & reference & reference states & sound & sound \& complete \\
    \midrule
    GnuTLS\_3.3.12\_client\_full & 8 & 12 & chat & 8 & False & False \\
    GnuTLS\_3.3.12\_client\_full & 8 & 12 & domain & 7 & True & False \\
    GnuTLS\_3.3.12\_client\_full & 8 & 12 & passive & 7 & False & False \\
    GnuTLS\_3.3.12\_client\_full & 8 & 12 & sound \& complete & 8 & True & True \\
    \midrule
    GnuTLS\_3.3.12\_server\_full & 9 & 12 & chat & 9 & False & False \\
    GnuTLS\_3.3.12\_server\_full & 9 & 12 & domain & 9 & True & False \\
    GnuTLS\_3.3.12\_server\_full & 9 & 12 & passive & 7 & False & False \\
\bottomrule
    \end{tabular}}
\end{table}

\begin{table}[htbp]
    \centering
    \caption{(Continued) Information on the TLS benchmark and references.}
    \label{tab:soundcompltable2}
    \resizebox{0.85\textwidth}{!}{
\begin{tabular}{lrr|rrrr}
    \toprule
    model & states & inputs & reference & reference states & sound & sound \& complete  \\
    \midrule
    GnuTLS\_3.3.8\_server\_regular & 11 & 8 & chat & 7 & False & False \\
    GnuTLS\_3.3.8\_server\_regular & 11 & 8 & domain & 6 & True & False \\
    GnuTLS\_3.3.8\_server\_regular & 11 & 8 & passive & 8 & False & False \\
    GnuTLS\_3.3.8\_server\_regular & 11 & 8 & sound \& complete & 10 & True & True \\ \midrule
    OpenSSL\_1.0.1g\_client\_regular & 9 & 7 & chat & 7 & False & False \\
    OpenSSL\_1.0.1g\_client\_regular & 9 & 7 & domain & 5 & True & False \\
    OpenSSL\_1.0.1g\_client\_regular & 9 & 7 & passive & 9 & False & False \\
\bottomrule
    \end{tabular}}
\end{table}

\newpage
\section{Full Results}
\begin{table}[htbp]
    \caption{Results Experiment 1 for the ASML benchmark. The last two columns show the median and standard deviation for the total number of symbols.}
    \resizebox{.99\textwidth}{!}{
        \centering
        \begin{tabular}{llll|rrrrr|rr}
            \toprule
            Model &  States &  Inputs &   Algorithm & Learned &  Learned States &  Learn Symbols &  Conformance Symbols &  Total Symbols &  Total Symbols Median &  Total Symbols Std \\
            \midrule
            m54 & 27 & 10 & $\lsharp$ & 100.0\% & 27 & 7018 & 135744 & 142762 & 127074 & 74596 \\
            m54 & 27 & 10 & $\erlsharp$ & 100.0\% & 27 & 1575 & 11235 & 12810 & 12204 & 2762 \\
            m54 & 27 & 10 & $\sclsharp$ & 100.0\% & 27 & 329 & 0 & 329 & 329 & 0 \\
            \midrule
            m41 & 25 & 47 & $\lsharp$ & 46.7\% & 24 & 47049 & 799740 & 846789 & 1000000 & 213277 \\
            m41 & 25 & 47 & $\erlsharp$ & 100.0\% & 25 & 5819 & 24487 & 30306 & 22943 & 19536 \\
            m41 & 25 & 47 & $\sclsharp$ & 100.0\% & 25 & 423 & 0 & 423 & 423 & 0 \\
            \midrule
            m199 & 27 & 56 & $\lsharp$ & 0.0\% & 16 & 44583 & 955453 & 1000036 & 1000006 & 160 \\
\bottomrule
            \end{tabular}
        }
    \end{table}

\begin{table}[htbp]
    \caption{Results Experiment 1 for the monitor benchmark. The last two columns show the median and standard deviation for the total number of symbols.}
    \resizebox{.99\textwidth}{!}{
        \centering
        \begin{tabular}{llllll|rrrrr|rr}
            \toprule
                  Name & Horizon &       Parameters &  States &  Inputs &   Algorithm & Learned &  Learned States &  Learn Symbols &  Conformance Symbols &  Total Symbols &  Total Symbols Median &  Total Symbols Std \\
            \midrule
            airportA-7 & 8 & ('DMAX=3', 'PMAX=3') & 46 & 18 & $\lsharp$ & 0.0\% & 29 & 33140 & 966866 & 1000006 & 1000007 & 4 \\
airportA-7 & 8 & ('DMAX=3', 'PMAX=3') & 46 & 18 & $\erlsharp$ & 100.0\% & 46 & 5384 & 59066 & 64450 & 59911 & 18270 \\
airportA-7 & 8 & ('DMAX=3', 'PMAX=3') & 46 & 18 & $\sclsharp$ & 100.0\% & 46 & 1404 & 6530 & 7934 & 7571 & 4846 \\
\midrule
airportB-7 & 8 & ('DMAX=3', 'PMAX=3') & 46 & 18 & $\lsharp$ & 0.0\% & 29 & 34165 & 965840 & 1000005 & 1000005 & 5 \\
airportB-7 & 8 & ('DMAX=3', 'PMAX=3') & 46 & 18 & $\erlsharp$ & 100.0\% & 46 & 7072 & 85688 & 92760 & 88113 & 27382 \\
airportB-7 & 8 & ('DMAX=3', 'PMAX=3') & 46 & 18 & $\sclsharp$ & 100.0\% & 46 & 1478 & 7700 & 9178 & 8409 & 3799 \\
\midrule
airportA-7 & 9 & ('DMAX=3', 'PMAX=3') & 50 & 18 & $\lsharp$ & 0.0\% & 29 & 31619 & 968390 & 1000009 & 1000008 & 8 \\
\bottomrule
            \end{tabular}}
\end{table}

\begin{table}[htbp]
    \caption{Results Experiment 2 for the TLS benchmark. The last two columns show the median and standard deviation for the total number of symbols.}
    \centering
    \resizebox{.85\textwidth}{!}{
        
        \begin{tabular}{llll|rrrrr|rr}
            \toprule
             Model &  States &  Inputs & Reference & Learned &  Learned States &  Learn Symbols &  Conformance Symbols &  Total Symbols &  Total Symbols Median &  Total Symbols Std \\
            \midrule
            OpenSSL\_1.0.1j\_client\_regular.dot & 6 & 7 & $\erlsharp()$ & 100.0\% & 6 & 295 & 1731 & 2027 & 833 & 2298 \\
            OpenSSL\_1.0.1j\_client\_regular.dot & 6 & 7 & $\eralsharp(\K_{P})$ & 100.0\% & 6 & 229 & 0 & 229 & 227 & 5 \\
            OpenSSL\_1.0.1j\_client\_regular.dot & 6 & 7 & $\eralsharp(\K_{GPT})$ & 100.0\% & 6 & 215 & 174 & 390 & 373 & 71 \\
            OpenSSL\_1.0.1j\_client\_regular.dot & 6 & 7 & $\eralsharp(\K_{M})$ & 100.0\% & 6 & 191 & 204 & 395 & 381 & 83 \\
            OpenSSL\_1.0.1j\_client\_regular.dot & 6 & 7 & $\soundlsharp(\K_{M})$ & 100.0\% & 6 & 140 & 85 & 225 & 223 & 29 \\
            OpenSSL\_1.0.1j\_client\_regular.dot & 6 & 7 & $\sclsharp(\K_{\sound\complete})$ & 100.0\% & 6 & 41 & 0 & 41 & 41 & 0 \\
            \midrule
            OpenSSL\_1.0.1l\_client\_regular.dot & 6 & 7 & $\erlsharp()$ & 100.0\% & 6 & 296 & 1731 & 2027 & 808 & 2298 \\
\bottomrule
            \end{tabular}}
\end{table}

\begin{table}[htbp]
    \caption{Results Experiment 3 for the TLS benchmark. The last two columns show the median and standard deviation for the total number of symbols.}
    \centering
    \resizebox{.9\textwidth}{!}{
    \begin{tabular}{llll|rrr|rrrr|rr}
        \toprule
    Model & States &  Algorithm & Reference Type & Correct (\%) & Incorrect (\%) & Violation (\%) & Learned States & Learn Symbols & Conformance Symbols & Total Symbols &  Total Symbols Median &  Total Symbols Std\\
\midrule
GnuTLS\_3.3.8\_client\_regular.dot & 6 & $\eralsharp$ & $\K_{MUT}$ & 100 & 0 & 0 & 6 & 255 & 371 & 627 & 598 & 176 \\
GnuTLS\_3.3.8\_client\_regular.dot & 6 & $\eralsharp$ & $\K_{P}$ & 100 & 0 & 0 & 6 & 282 & 197 & 479 & 435 & 144 \\
GnuTLS\_3.3.8\_client\_regular.dot & 6 & $\eralsharp$ & $\K_{M}$ & 100 & 0 & 0 & 6 & 269 & 283 & 552 & 515 & 130 \\
GnuTLS\_3.3.8\_client\_regular.dot & 6 & $\eralsharp$ & $\K_{\sound\complete}$ & 100 & 0 & 0 & 6 & 305 & 0 & 305 & 302 & 10 \\
GnuTLS\_3.3.8\_client\_regular.dot & 6 & $\soundlsharp$ & $\K_{MUT}$ & 100 & 0 & 0 & 6 & 123 & 30 & 153 & 147 & 24 \\
GnuTLS\_3.3.8\_client\_regular.dot & 6 & $\soundlsharp$ & $\K_{P}$ & 0 & 100 & 0 & 8 & 175 & 33 & 208 & 206 & 23 \\
GnuTLS\_3.3.8\_client\_regular.dot & 6 & $\soundlsharp$ & $\K_{M}$ & 100 & 0 & 0 & 6 & 167 & 109 & 276 & 254 & 56 \\
\bottomrule
\end{tabular}}
\end{table}

\begin{table}[htbp]
    \caption{(Continuation) Results Experiment 3 for the TLS benchmark. The last two columns show the median and standard deviation for the total number of symbols.}
    \centering
    \resizebox{.9\textwidth}{!}{
    \begin{tabular}{llll|rrr|rrrr|rr}
        \toprule
    Model & States & Algorithm & Reference Type & Correct (\%) & Incorrect (\%) & Violation (\%) & Learned States & Learn Symbols & Conformance Symbols & Total Symbols & Total Symbols Median &  Total Symbols Std \\
    \midrule
    GnuTLS\_3.3.12\_client\_full.dot & 8 & $\eralsharp$ & $\K_{MUT}$ & 100 & 0 & 0 & 8 & 470 & 83 & 553 & 532 & 63 \\
GnuTLS\_3.3.12\_client\_full.dot & 8 & $\eralsharp$ & $\K_{P}$ & 100 & 0 & 0 & 8 & 424 & 329 & 753 & 730 & 157 \\
GnuTLS\_3.3.12\_client\_full.dot & 8 & $\eralsharp$ & $\K_{M}$ & 100 & 0 & 0 & 8 & 366 & 667 & 1033 & 1014 & 266 \\
GnuTLS\_3.3.12\_client\_full.dot & 8 & $\eralsharp$ & $\K_{\sound\complete}$ & 100 & 0 & 0 & 8 & 460 & 0 & 460 & 462 & 17 \\
GnuTLS\_3.3.12\_client\_full.dot & 8 & $\soundlsharp$ & $\K_{MUT}$ & 100 & 0 & 0 & 8 & 100 & 66 & 166 & 154 & 56 \\
GnuTLS\_3.3.12\_client\_full.dot & 8 & $\soundlsharp$ & $\K_{P}$ & 0 & 100 & 0 & 6 & 58 & 0 & 58 & 58 & 0 \\
GnuTLS\_3.3.12\_client\_full.dot & 8 & $\soundlsharp$ & $\K_{M}$ & 100 & 0 & 0 & 8 & 239 & 238 & 477 & 446 & 168 \\
\bottomrule
\end{tabular}}
\end{table}}

\end{document}